\documentclass{article}

\PassOptionsToPackage{numbers, compress}{natbib}
\usepackage[preprint]{neurips_2026}

\usepackage[utf8]{inputenc}
\usepackage[T1]{fontenc}
\usepackage{url}
\usepackage{booktabs}
\usepackage{multirow}
\usepackage{colortbl}
\usepackage{makecell}

\newcommand{\msd}[2]{\makecell{#1\\[-2pt]{\fontsize{5pt}{5pt}\selectfont$\pm\,$#2}}}
\newcommand{\msdb}[2]{\makecell{\textbf{#1}\\[-2pt]{\fontsize{5pt}{5pt}\selectfont$\pm\,$#2}}}
\usepackage{array}
\usepackage{amsfonts}
\usepackage{amsmath}
\usepackage{amssymb}
\usepackage{amsthm}
\usepackage{mathtools}

\newtheorem{theorem}{Theorem}[section]
\newtheorem{lemma}[theorem]{Lemma}

\newtheorem{proposition}[theorem]{Proposition}
\theoremstyle{definition}

\newtheorem{assumption}[theorem]{Assumption}
\theoremstyle{remark}

\theoremstyle{plain}
\usepackage{nicefrac}
\usepackage{microtype}
\usepackage{xcolor}
\usepackage{graphicx}
\usepackage{float}
\usepackage{tikz}
\usetikzlibrary{arrows.meta, positioning, calc, fit, backgrounds, shapes.misc}
\usepackage{enumitem}
\usepackage{hyperref}
\setlist[enumerate]{leftmargin=1.4em, itemsep=2pt, topsep=3pt, parsep=1pt}
\setlist[itemize]{leftmargin=1.2em, itemsep=2pt, topsep=3pt, parsep=1pt}
\setlist[description]{leftmargin=1.4em, itemsep=2pt, topsep=3pt, parsep=1pt, style=nextline}
\newcommand{\figorplaceholder}[2][\linewidth]{%
  \IfFileExists{#2}%
    {\includegraphics[width=#1]{#2}}%
    {\fbox{\parbox{0.85\linewidth}{\centering\vspace{1.5cm}figure placeholder\vspace{1.5cm}}}}%
}

\newcommand{\bpd}{\mathrm{bpd}}
\newcommand{\RR}{\mathbb{R}}
\newcommand{\EE}{\mathbb{E}}
\newcommand{\Var}{\mathrm{Var}}
\newcommand{\KL}{\mathrm{KL}}
\newcommand{\Cov}{\mathrm{Cov}}
\DeclareMathOperator*{\argmin}{arg\,min}
\newcommand{\lh}{\ell,h}

\newcommand{\kvmse}{KV-COBRA\ensuremath{_{\mathsf{MSE}}}}
\newcommand{\kvkl}{KV-COBRA\ensuremath{_{\mathsf{KL}}}}
\newcommand{\kvc}{KV-COBRA\,w/o\,C2}

\title{KV-COBRA: KV Cache Compression\\
via Co-Optimized Bit-Rank Allocation}

\author{%
  Sihyeon Ha \qquad Jaeho Lee \qquad Yo-Seb Jeon\thanks{Corresponding author} \\[2pt]
  Pohang University of Science and Technology (POSTECH) \\[2pt]
  \texttt{\{sihyeon.ha, jaeho.lee, yoseb.jeon\}@postech.ac.kr}
}

\begin{document}

\maketitle

\begin{abstract}
What limits KV-cache compression at extreme bit-rates?
We argue that it is not the choice of compression scheme, but
how its budget is allocated across attention heads.
Existing methods apply rank and bit-width uniformly,
ignoring that each head has a different optimal mix of
rank truncation and quantization.
We show that \emph{co-optimizing rank and bit-width per
head}---using only standard low-rank projection and scalar
quantization---dominates uniform allocation, with the
largest gains at low bit-rates. Our method, \textbf{KV-COBRA}
(\emph{Co-Optimized Bit-Rank Allocation}), formalizes this
as a resource-allocation problem: it balances
rank-truncation loss against quantization loss within each
head, then redistributes budget across heads to minimize
total distortion. A fused Hadamard rotation equalizes
per-channel variance, and reordering the SVD basis by
attention-KL importance makes the solver query-aware. The
same allocator extends to joint $K{+}V$ compression.
On perplexity, zero-shot, and long-context benchmarks from
$0.5$ to $4$ bits per dimension (bpd), KV-COBRA shows the
smallest accuracy degradation among evaluated methods at
low bpd, with no per-token overhead.
\end{abstract}

\section{Introduction}
\label{sec:intro}

Generating each token with a large language model (LLM) requires loading the entire key-value (KV) cache. Under long-context scenarios, the size of these caches exceeds that of the model weights, rendering the decoding to be \emph{memory-bandwidth-bound}: Loading the KV cache becomes the bottleneck, not computing on it~\citep{pope2022, yuan2024}. To alleviate this issue, two strategies to compress the KV cache have been proposed: \emph{Eviction} discards KV for unimportant or redundant tokens~\citep{h2o, snapkv, streamingllm}. \emph{Per-token compression}, our focus, keeps every
token but stores each at lower \emph{bits per dimension} (bpd), via quantization~\citep{kivi, kvquant, gear}, low-rank projection~\citep{svdq, kqsvd, turboquant}, or both. The strongest per-token compressors combine these, applying top-$r$ SVD truncation followed by $b$-bit quantization on the retained coordinates.

In these per-token compression methods, the rank $r$ and quantization bit-width $b$ are typically set as global hyperparameters, applied uniformly across \textit{all attention heads}. Yet, heads differ dramatically in how their cache spectrum concentrates; some place most of their energy in a handful of directions, while others spread it widely
(Figure~\ref{fig:hero}). Thus, applying a single $r$ is wasteful for concentrated heads and inadequate for diffuse ones. 
Likewise, using a uniform bit-width $b$ across heads may be suboptimal. Such homogeneous allocation may limit accuracy, especially in low-bpd regimes.

This raises two questions. First, \emph{how should the total rate budget be allocated across heads?} Some heads may need higher
precision, while others tolerate aggressive compression. Second,
\emph{how should each head's allotted rate split between $r$ and $b$?} The two questions are coupled, since the optimal
$r$ for a head depends on its $b$ and vice versa, so
optimizing one while holding the other fixed cannot reach
the joint optimum.

We view this as a classical rate-distortion allocation problem~\citep{shannon1959, berger1971}. Each head has a per-head distortion curve---the best achievable reconstruction error at a given rate jointly over
$r$ and $b$---and the global problem is to distribute total rate across heads to minimize the sum of distortions. This is the setting in which information-theoretic water-filling applies~\citep{cover2006}. The resulting first-order conditions yield a per-head rule for splitting the rate between $r$ and $b$ and a marginal-distortion-equalizing schedule across heads, both solvable by a few iterations at calibration time.

We instantiate this view as \textbf{KV-COBRA}
(\emph{Co-Optimized Bit-Rank Allocation}), a two-level
optimizer that runs entirely at calibration time. The inner
level (\textbf{C1}) selects each head's
$(r^\star, b^\star)$ by equating the marginal costs of
truncation and quantization on its eigenvalue spectrum
(Figure~\ref{fig:hero}a). The outer level (\textbf{C2})
redistributes the total budget across heads via
marginal-slope equalization. A Hadamard rotation fused into
the SVD basis~\citep{turboquant, quarot} equalizes the
steeply non-uniform per-channel variance of the retained
coordinates at zero cost (Figure~\ref{fig:hero}c). To make
the objective query-aware rather than purely
reconstruction-driven, we derive a second-order attention-KL
bound that weights each direction by how much the quantizer
hurts it and how much the attention map cares about it;
reordering the basis by this weight before running C1/C2
makes the allocation query-aware. KV-COBRA reuses the same compression kernel as
prior methods; what changes is how the bit-budget is
allocated.

Across LLaMA-3.1-8B, Mistral-7B-v0.3, and
Qwen2.5-7B-Instruct, KV-COBRA matches or beats every prior
SVD- and rotation-based compressor at $0.5$--$2$~bpd on
Wikitext-2/PTB/C4 perplexity, LongBench F1, and
zero-shot reasoning, and the gap widens at the lowest rates
where allocation matters most. The takeaway is one line:
\emph{the bottleneck of low-bit KV compression is allocation
across heads, not the kernel itself}.

\begin{figure}[!t]
\centering
\figorplaceholder[0.95\linewidth]{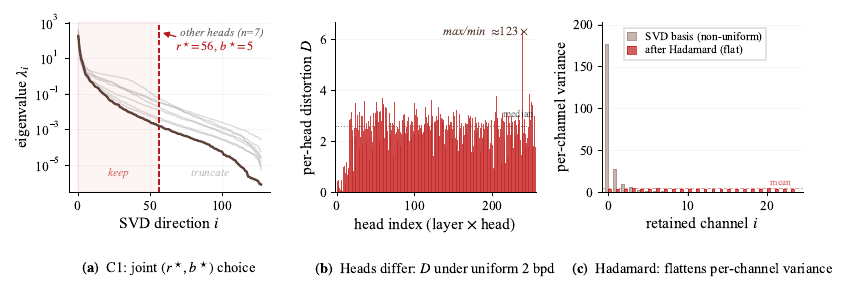}
\caption{\textbf{(a)~Spectrum sets $(r^\star, b^\star)$:}
one head (dark) against the layer's 7 others (grey).
\textbf{(b)~Heads differ dramatically:} per-head $D$ at
uniform 2~bpd spans $100\times$ across 256 KV heads.
\textbf{(c)~Hadamard flattens variance:} per-channel,
retained subspace, before (beige) and after (red).}
\label{fig:hero}
\vspace{-5mm}
\end{figure}

Our contributions can be summarized as follows:
\begin{enumerate}[topsep=0pt,parsep=0pt]
\item \textbf{A joint rank-bit optimizer with optimality
conditions} (Section~\ref{sec:method}). C1 selects each
head's $(r, b)$ optimum, and C2 redistributes the total
budget across heads. Both are gradient-free, run at
calibration time, and add no per-token overhead.

\item \textbf{A query-aware upgrade via attention-KL
reordering} (Section~\ref{sec:kl}). A second-order bound on
attention-KL yields a per-direction importance weight.
Reordering the SVD basis by this weight before C1/C2 makes
the solver optimize attention quality rather than
reconstruction MSE.

\item \textbf{Graceful sub-$1$~bpd compression with no new
kernel} (Section~\ref{sec:experiments}). KV-COBRA stays
accurate down to $0.5$~bpd where uniform-allocation baselines
collapse, and the same allocator extends to joint $K{+}V$
compression (Tables~\ref{tab:kv-joint},~\ref{tab:kv-joint-lb}).
\end{enumerate}

\vspace{-1mm}

\section{Related work}
\label{sec:related}

\paragraph{KV cache quantization.}
Existing KV-cache quantization methods
reduce memory footprint through improved quantization
and reconstruction.
KIVI~\citep{kivi} quantizes $K$ per-channel and $V$ per-token
with grouped bit widths; KVQuant~\citep{kvquant} adds outlier handling and
pre-RoPE~\citep{rope} key quantization;
GEAR~\citep{gear} supplements grouped uniform quantization
with a power-iteration low-rank residual.
AQUA-KV~\citep{aquakv} predicts each layer's cache from the
previous one and quantizes the residual at $2$--$2.5$ bits
per value. All of these methods
operate at a uniform per-head bit budget.

\paragraph{Rotation-based quantization.} A second family first
projects the cache into a basis that concentrates variance
before quantizing: SVDq~\citep{svdq} uses a per-layer SVD basis
with an 8-group mixed-precision schedule; TurboQuant~\citep{turboquant}
applies a random rotation plus per-channel Lloyd--Max
quantization~\citep{lloyd1982, max1960};
QuaRot~\citep{quarot} fuses Hadamard rotations into the weights.
KQ-SVD~\citep{kqsvd} instead
uses a query-aware basis derived from the joint $K$/$Q$
statistics. All fix rank and per-head budget as
hyperparameters. Random-rotation scalar quantization with
distortion guarantees predates KV compression, with
EDEN~\citep{eden} for distributed mean estimation and
RaBitQ~\citep{rabitq} for nearest-neighbor search. KV-COBRA
adopts this standard kernel unchanged. What it adds is the
allocation on top.

\paragraph{Eviction.} A complementary line of work reduces
cache size by discarding entries rather than compressing
them, using importance scores such as accumulated attention
(H2O~\citep{h2o}) or context-reconstruction loss
(KVzip~\citep{kvzip}). These methods operate on the token
axis and compose with per-token compressors.

\paragraph{Architectural KV-cache reduction.}
Grouped-query attention~\citep{gqa}, multi-head latent
attention~\citep{deepseekv2}, cross-layer
attention~\citep{cla}, and dynamic memory
compression~\citep{dmc} share, merge, or compress KV at
training time, whereas KV-COBRA is post-training, drop-in.

\paragraph{Allocation taxonomy.}
Table~\ref{tab:crosswalk} organizes prior work by which
allocation axes each method addresses; to our knowledge no
prior method optimizes more than one. Throughout the paper we
plot the \emph{effective bpd}
$\bpd_{\text{eff}} = \bpd_{\text{nominal}} + \Delta_{\text{method}}$
from the $\Delta$ column (Appendix~\ref{app:effective-bpd}).

\begin{table}[t]
\centering
\caption{\textbf{Allocation taxonomy for KV-cache
compressors.} \checkmark~= optimized, fix = hyperparameter,
eq.\ = equalized by rotation, --- = N/A. $\Delta\bpd$ is the
per-(channel, token) metadata overhead
(Appendix~\ref{app:effective-bpd}).}
\label{tab:crosswalk}
\small
\setlength{\tabcolsep}{4pt}
\begin{tabular}{lcccccc}
\toprule
Method & Basis & layer & head & channel & rank & $\Delta\bpd$ \\
\midrule
KIVI~\citep{kivi}          & identity         & fix & fix & fix & --- & $+1.0$ \\
KVQuant~\citep{kvquant}    & identity         & fix & fix & fix & --- & $+0.16$ \\
GEAR~\citep{gear}          & identity         & fix & fix & fix & --- & $+1.0$ \\
TurboQuant~\citep{turboquant} & rand.\ rotation & fix & fix & eq. & --- & $0$ \\
QuaRot~\citep{quarot}      & Hadamard (fused) & fix & fix & eq. & --- & $0$ \\
SVDq~\citep{svdq}          & per-layer SVD    & fix & fix & fix & fix & $0$ \\
KQ-SVD~\citep{kqsvd}       & query-aware SVD  & $\checkmark$ & fix & --- & adaptive ($\varepsilon$) & $0$ \\
\midrule
\textbf{\kvkl{} (ours)} & per-head SVD + Had. & \textbf{$\checkmark$} & \textbf{$\checkmark$}
                           & eq. & \textbf{$\checkmark$} & $0$ \\
\bottomrule
\end{tabular}
\end{table}

\vspace{-1mm}

\section{KV-COBRA under reconstruction MSE}
\label{sec:method}
\vspace{-1mm}
This section derives the two-level bit-allocation solver under
$L_{2}$ reconstruction MSE (Figure~\ref{fig:pipeline}).
Section~\ref{sec:kl} then upgrades the objective to attention-KL,
making the solver query-aware.

\paragraph{Setup and notation.} We consider a transformer with $L$ layers,
$H$ KV heads with head dimension $d$. Given a context of length $T$, the model requires storing $2LHTd$ values as its KV cache. We denote by
$K_{\lh}$, $V_{\lh}\in\RR^{T\times d}$ the per-(layer, head)
$K$/$V$ matrices, by
$\lambda_{\lh,1}\geq\cdots\geq\lambda_{\lh,d}$ the descending
eigenvalues of their centered covariances, and by
$U_{\lh}\in\RR^{d\times d}$ the corresponding SVD eigenbases.
We drop the side label when the argument applies to both $K$
and $V$. We write $\bpd$ for the average bits per dimension
and $B_{\mathrm{tot}}=LHd\,\bpd$ for the total bit budget.

\paragraph{Distortion model.}
KV-COBRA's Hadamard rotation mixes
the retained SVD coordinates, pushing each channel
toward a Gaussian marginal.
For a zero-mean Gaussian channel with variance $\sigma^{2}$,
a $b$-bit uniform scalar quantizer has MSE
$\frac{\sigma^{2}}{12}2^{-2b} + o(2^{-2b})$
(Bennett~\citep{bennett1948};
Appendix~\ref{app:bennett}).
We adopt the leading-order distortion
\begin{equation}
q(b)\;\coloneqq\;\tfrac{1}{12}\cdot 2^{-2b},
\label{eq:q}
\end{equation}
However, a single fixed transform
does not guarantee Gaussianity, and known bounds
are asymptotic in dimension~\citep{benbasat2026}.
We therefore treat Gaussianity as a working assumption
and validate it empirically at our head dimensions
(Appendix~\ref{app:gaussianity}).
 
Projecting a head onto its top-$r$ SVD directions discards
variance $\sum_{i>r}\lambda_i$
(Eckart--Young~\citep{eckart1936}) and quantizing the $r$
retained coordinates at $\bar b$ bits adds
$q(\bar b)\sum_{i\leq r}\lambda_i$.
At the per-head level the
compressor's effect on generation quality can be modeled by
two natural distortion choices --- the reconstruction MSE in the
SVD basis, or the attention-KL divergence between clean and
compressed softmax distributions --- which both collapse to a
common closed form
\begin{equation}
D(r, \bar b)
\;=\;
\underbrace{\sum_{i>r}w_{i}}_{\text{projection loss}}
\;+\;\underbrace{q(\bar b)\sum_{i\leq r}w_{i}}_{\text{quantization loss}},
\qquad
w_i = \begin{cases}
\lambda_i & \text{(MSE objective)},\\[2pt]
\sigma_{Q,i}^{2}\,\sigma_{K,i}^{2} & \text{(attention-KL objective)},
\end{cases}
\label{eq:drb-preview}
\end{equation}
differing only in whether each direction is weighted by the key
eigenvalue or the query--key variance product (derivation of
the KL form in Section~\ref{sec:kl}). Both weightings plug
into the same two-level C1/C2 solver; this section treats the
MSE case (simpler, query-free), Section~\ref{sec:kl} upgrades
to KL.

\begin{figure}[t!]
\centering
\figorplaceholder[\linewidth]{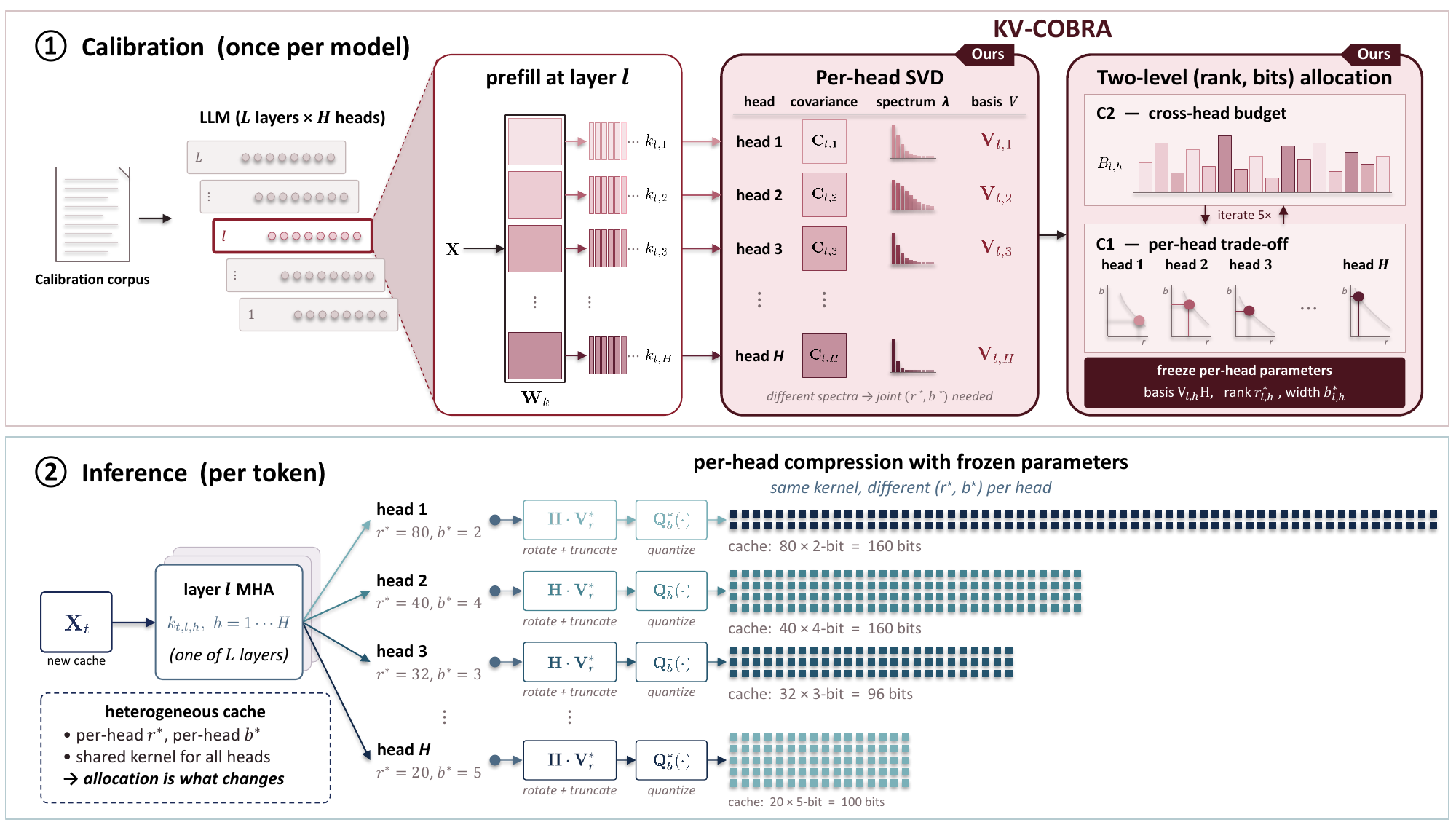}
\vspace{-0.5em}
\caption{KV-COBRA pipeline. \textbf{Calibration.} Per-head
SVD of prefill keys. A two-level C1+C2 allocator freezes
$(V_{\ell,h}, r^\star_h, b^\star_h)$. \textbf{Inference.}
Each key uses $H_{r^\star}\!\cdot\! V_{\ell,h}^\top$ and a
$b^\star_h$-bit quantizer.}
\label{fig:pipeline}
\vspace{-5mm}
\end{figure}

\subsection{Problem formulation}
\label{sec:framing}

KV-COBRA jointly optimizes the retained rank and
bit budget of each head, rather than fixing them as
global hyperparameters.
Figure~\ref{fig:hero} illustrates this trade-off
for LLaMA-3.1-8B.
We use the MSE distortion $D^{\mathrm{MSE}}$
of Eq.~\eqref{eq:drb-preview}; the KL formulation
in Section~\ref{sec:kl} differs only in replacing
$\lambda_{\lh,i}\mapsto w_{\lh,i}$.
The resulting optimization problem is
\begin{equation}
\min_{\{(r_{\lh},\,\bar b_{\lh})\}}
\;\sum_{\ell,h}D_{\lh}\bigl(r_{\lh},\,\bar b_{\lh}\bigr)
\qquad\text{s.t.}\qquad
\sum_{\ell,h}r_{\lh}\,\bar b_{\lh}\leq B_{\mathrm{tot}},
\label{eq:global}
\end{equation}
with $D_{\lh}$ from Eq.~\eqref{eq:drb-preview} (the MSE choice).
Problem~\eqref{eq:global} has a natural two-level decomposition:
\begin{description}
\item[Inner, per head (C1).] Given the head's budget
$B_{\lh}$, jointly choose the retained rank $r_{\lh}$ and
bit width $\bar b_{\lh}=B_{\lh}/r_{\lh}$ that minimize
$D_{\lh}$.
\item[Outer, across heads (C2).] Given the total budget
$B_{\mathrm{tot}}$, distribute it across heads as $\{B_{\lh}\}$
so that the total distortion $\sum_{\lh} D^{\star}_{\lh}(B_{\lh})$
is minimized.
\end{description}

Both levels admit first-order optimality conditions
(Propositions~\ref{thm:c1-kkt},~\ref{thm:c2-kkt}); we solve C1 by
integer enumeration and C2 by a damped iteration that
converges within five rounds on every tested model. The
compression kernel is inherited from the rotate-and-quantize
family~\citep{svdq, turboquant}, so with both levers frozen
KV-COBRA reduces to a basic SVD pipeline.
Sections~\ref{sec:c1}--\ref{sec:c3} derive each level in
turn.

\subsection{Per-head optimum: rank and bit width together}
\label{sec:c1}

Fix a single (layer, head). With a budget $r\cdot\bar b\leq B$,
increasing $r$ reduces the projection loss $\sum_{i>r}\lambda_i$
but spreads the budget thinner; the optimum balances the two.
Substituting $\bar b=B/r$ in Eq.~\eqref{eq:drb-preview} reduces
$D$ to a one-dimensional discrete objective
$f(r)\coloneqq D(r,B/r)$, and on the even-integer rank grid
(step $2$) the forward difference is
\begin{equation}
\Delta f(r) = -(\lambda_{r+1}{+}\lambda_{r+2})
\bigl(1-q\bigl(\tfrac{B}{r+2}\bigr)\bigr)
+\bigl(q\bigl(\tfrac{B}{r+2}\bigr)-q\bigl(\tfrac{B}{r}\bigr)\bigr)
\sum_{i\leq r}\lambda_i,
\label{eq:c1-foc}
\end{equation}
with any integer optimum $r^{\star}$ satisfying
$\Delta f(r^{\star}-2)\leq 0\leq\Delta f(r^{\star})$
(Proposition~\ref{thm:c1-kkt}): the projection-loss gain from
two more directions matches the quantization-loss cost on the
retained ones. We solve by enumerating
$r\in\{2, 4,\dots,\min(d, \lfloor B/b_{\min}\rfloor)\}$,
setting $\bar b=\mathrm{round}(B/r)$ clipped to $[2, 8]$, and
returning the $(r, \bar b)$ with the smallest $D$ --- an
$O(d/2)$ search per head, calibration-only. On typical
calibration runs at $2$~bpd, $r^{\star}$ spans $\sim$$[16, 100]$
and $b^{\star}\in[2, 5]$ across heads
(Fig.~\ref{fig:rstar-div}).

\subsection{Cross-layer redistribution: sum-distortion minimization}
\label{sec:c2}

Heads differ: heavy-tailed ones saturate their budget early,
flat ones stay starved. C2 uses C1 as a black-box oracle
$D_{\lh}^\star(B_{\lh})\coloneqq\min_{rb\leq B_{\lh}}D_{\lh}(r, b)$
and redistributes so that every head sits at the same marginal
distortion:
\begin{equation}
\min_{\{B_{\lh}\}}\sum_{\lh}D_{\lh}^\star(B_{\lh})
\quad\text{s.t.}\quad
\sum_{\lh}B_{\lh} = B_{\mathrm{tot}},\;\;B_{\lh}\geq B_{\min}.
\label{eq:c2}
\end{equation}
Under standard convexity assumptions on $D_{\lh}^\star$
(Appendix~\ref{app:c2}, Assumption~\ref{ass:disc-convex}),
the Lagrangian stationarity reduces to
$(D_{\lh}^\star)'(B_{\lh})=-\nu$, i.e.,\ marginal
distortions equalize at the optimum. We solve~\eqref{eq:c2}
with a damped proportional update on $B_{\lh}$
(Appendix~\ref{app:c2}). An exact DP solver attains the
same PPL within noise but at substantially higher cost
(Appendix~\ref{app:c2-dp}); we adopt the damped update.

\subsection{Per-channel equalization via Hadamard rotation}
\label{sec:c3}

The $r$ retained SVD coordinates have steeply decaying
variances (up to $10^{3}\times$; Figure~\ref{fig:hero}c), so a
per-channel quantizer would need an $L{\times}H$ grid of
step-size schedules. We instead fuse a Hadamard rotation
$V_r\to H_r V_r$, with $H_r$ the $r{\times}r$ Walsh--Hadamard
matrix and a fixed random sign flip drawn once at calibration
and shared across all tokens~\citep{turboquant, quarot}.
Absorbed into the same matrix--vector multiply
(\emph{zero latency}), it equalizes per-channel variance --- the
$10^{3}\times$ spread collapses to $O(1)$, so a single uniform
$\bar b$-bit quantizer suffices --- and pushes each channel
toward Gaussian, the regime where Bennett's model
(Appendix~\ref{app:bennett}) is tightest (validated in
Appendix~\ref{app:gaussianity}).

\begin{table}[ht!]
\caption{KV-COBRA pipeline: one-time calibration (1--6), then per-token inference (7--8).}
\label{alg:kvrobin}
\centering
\small
\setlength{\tabcolsep}{3pt}
\begin{tabular}{@{}rp{0.82\linewidth}@{}}
\toprule
\multicolumn{2}{l}{\textbf{Calibration} (once per model; input: calibration data, target $\bpd$)} \\
\midrule
1 & Per-head $K$ covariance $\to \{\lambda_{\lh,i}\},\{V_{\lh}\}$; query variance $\{\sigma_{Q,\lh,i}^2\}$ in the SVD basis \\
2 & $w_{\lh,i} \leftarrow \sigma_{Q,\lh,i}^2 \cdot \lambda_{\lh,i}$;\; reorder $V_{\lh}$ by $w_{\lh,i}$ descending \hfill\textrm{[\S\ref{sec:kl}]} \\
3 & $B_{\lh} \leftarrow \bpd \cdot d$ for all $(\ell,h)$ \\
4 & \textbf{5 rounds:} $(r^\star_{\lh}, \bar b^\star_{\lh}) \leftarrow \text{C1}(w_{\lh}, B_{\lh})$;\; $B_{\lh} \leftarrow \text{C2-update}$ \hfill\textrm{[\S\ref{sec:c1},\,\S\ref{sec:c2}]} \\
5 & Fuse Hadamard: $V_{\lh} \leftarrow H_{r^\star}\!\cdot\! V_{\lh}[:,\,:r^\star]$ \hfill\textrm{[\S\ref{sec:c3}]} \\
6 & Store $\{V_{\lh},\; r^\star_{\lh},\; \bar b^\star_{\lh}\}$ per head \\
\midrule
\multicolumn{2}{l}{\textbf{Inference} (per token, per head)} \\
\midrule
7 & $\hat{z}_t \leftarrow \mathrm{quantize}(V_{\lh}^\top K_t,\; \bar b^\star_{\lh})$ \hfill\textrm{[project+Hadamard as one multiply]} \\
8 & Cache $\hat{z}_t$;\; dequantize on attention read \\
\bottomrule
\end{tabular}
\end{table}

\section{KV-COBRA under attention-KL}
\label{sec:kl}

The $L_{2}$ formulation of Section~\ref{sec:method} minimizes
reconstruction error in the SVD basis, but downstream quality
depends on what the attention layer \emph{does} with the cache,
not on the cache itself. The $L_{2}$ objective weights every
direction by its $K$-eigenvalue $\lambda_i$, regardless of
whether the query cares about it.
Write the quantized key as $\tilde{K}_t=K_t+E_t$, the clean
attention distribution as $p=\mathrm{softmax}(s)$ with logits
$s_t=\langle Q,K_t\rangle/\sqrt{d}$, and the quantized
distribution as $\tilde{p}=\mathrm{softmax}(s+\Delta s)$.

\begin{proposition}[Attention-KL upper bound]
\label{prop:kl-bound}
Under the high-resolution additive-noise model
($\EE[E_t]{=}0$, channel-wise uncorrelated with
$\Var[E_{t,i}]{=}q(\bar b)\,\sigma_{K,i}^{2}$,
$Q{\perp}E_t$),
\begin{equation}
\EE\bigl[\KL(p\,\Vert\,\tilde{p})\bigr]
\;\underset{(a)}{\leq}\;
\tfrac{1}{2}\,\EE\bigl[\Delta s^{\top}\!F(p)\,\Delta s\bigr]
\;\underset{(b)}{\leq}\;
\tfrac{1}{4}\sum_{t=1}^{T}\Var[\Delta s_t]
\;\underset{(c)}{=}\;
C\!\sum_{i=1}^{r}
\underbrace{\sigma_{Q,i}^{2}\,\sigma_{K,i}^{2}}_{\eqqcolon\;w_i}
q(\bar b),
\label{eq:kl-obj}
\end{equation}
with $C=T/(4d)$. The relevant per-direction quality weight is
therefore $w_i=\sigma_{Q,i}^{2}\sigma_{K,i}^{2}$, not the
$K$-spectrum $\lambda_i$ alone.
\end{proposition}

\begin{enumerate}[label=\textbf{(\alph*)}]
\item \textbf{Taylor expansion of softmax KL}, with
$F(p)=\mathrm{diag}(p)-pp^{\top}$ the softmax Fisher information.
Used as a leading-order proxy; the cubic remainder is signed and
absorbed into our weight ranking, not the absolute KL value.
\item \textbf{Fisher tight bound.}
$F(p)\preceq\tfrac{1}{2}I$ (Popoviciu on unit vectors), giving
the $\tfrac{1}{4}$ leading constant after combining with the
Taylor $\tfrac{1}{2}$.
\item \textbf{Query--key factorization.}
$\Delta s_t=Q\!\cdot\!E_t^{\top}/\sqrt{d}$; assuming
$Q\perp E_t$ and channel-wise uncorrelated $\varepsilon_i$ gives
$\Var[\Delta s_t]=d^{-1}\sum_i\EE[Q_i^2]\,\EE[\varepsilon_i^2]$,
and with $\EE[Q_i^2]\eqqcolon\sigma_{Q,i}^2$ this yields the
per-direction weight
$w_i=\sigma_{Q,i}^{2}\sigma_{K,i}^{2}$. Summing over the
$T$ tokens absorbs a factor of $T$ into $C$.
\end{enumerate}

\paragraph{Reordering and plug-in.}
The weights $w_i$ are \emph{not} monotone in $\lambda_i$: a
direction with small $\lambda_i$ but large $\sigma_{Q,i}^{2}$
would be discarded by an MSE-ordered truncation the query
actually attends to. Before running C1/C2 we reorder the SVD basis by $w_i$
descending (once at calibration, negligible cost), so
``top-$r$'' means the $r$ directions with the largest
attention-KL importance. These
reordered weights replace $\lambda_i$ in $D^{\mathrm{KL}}$ of
Eq.~\eqref{eq:drb-preview}, and the forward difference of
Eq.~\eqref{eq:c1-foc} and the cross-head sum-distortion
envelope of Eq.~\eqref{eq:c2} carry over verbatim. The resulting
method is \kvkl{} (main); the $L_2$ version serves as ablation.
Table~\ref{alg:kvrobin} summarizes. The full C1+C2 calibration
takes under $200$~ms across all three models (Appendix~\ref{app:c2-dp}),
comparable to generating a few dozen tokens at standard FP16 decode rates.

\section{Experimental setup}
\label{sec:setup}
\label{sec:design}

\paragraph{Scope.} We compress $K$ pre-RoPE~\citep{rope};
post-RoPE the per-head covariance is position-dependent and
no static SVD basis captures it (Appendix~\ref{app:postrope}).
The body focuses on $k$-only because the $V$ spectrum is
structurally flatter (Appendix~\ref{app:vside}); joint
$K{+}V$ results follow in
Tables~\ref{tab:kv-joint} and~\ref{tab:kv-joint-lb}.

\paragraph{Models.} Three independently trained
$7$--$8$B grouped-query backbones: LLaMA-3.1-8B~\citep{llama3},
Mistral-7B-v0.3~\citep{mistral}, and
Qwen2.5-7B-Instruct~\citep{qwen25}, all loaded in \texttt{fp16}.

\paragraph{Benchmarks.}
(i)~Perplexity on Wikitext-2~\citep{wikitext}, PTB~\citep{ptb}, and
C4~\citep{c4} (sliding window $2{,}048$, $32$k evaluation tokens
per dataset);
(ii)~zero-shot accuracy on ARC-c~\citep{arc},
HellaSwag~\citep{hellaswag}, PIQA~\citep{piqa},
WinoGrande~\citep{winogrande}, MMLU~\citep{mmlu}
(200 samples per task);
(iii)~LongBench~\citep{longbench} F1 on five tasks
(NarrativeQA~\citep{narrativeqa}, QASPER~\citep{qasper},
MultiFieldQA-en, HotpotQA~\citep{hotpotqa},
MuSiQue~\citep{musique}; 50 samples per task).

\paragraph{Baselines and protocol.} We compare against
KIVI~\citep{kivi}, KVQuant~\citep{kvquant}, GEAR~\citep{gear},
TurboQuant~\citep{turboquant}, SVDq~\citep{svdq}, and
KQ-SVD~\citep{kqsvd}. The pipeline (calibration, compression,
evaluation) is deterministic given the seed; multi-seed cells
(Table~\ref{tab:main-ablation}, Fig.~\ref{fig:main-results})
report mean$\,\pm\,$1\,SD over $5$ seeds. Per-baseline
bit-width quirks, calibration details, and the full bit-rate
scan are in Appendix~\ref{app:setup}.

\section{Results}
\label{sec:experiments}

Per-head $(r^\star, b^\star)$ optima spread across many
operating points and shift with budget and model
(Fig.~\ref{fig:rstar-div}). A single fixed $(r, b)$ cannot
fit all heads, so per-head allocation is the axis on which
the methods below differ. C2 then equalizes per-head distortion across heads.
With uniform per-head budgets a
few heads carry most of the residual error, while the C2
allocation flattens the curve and reduces the cross-head
distortion std by $2$--$8\times$.

\begin{figure}[!ht]
\centering
\figorplaceholder[0.85\linewidth]{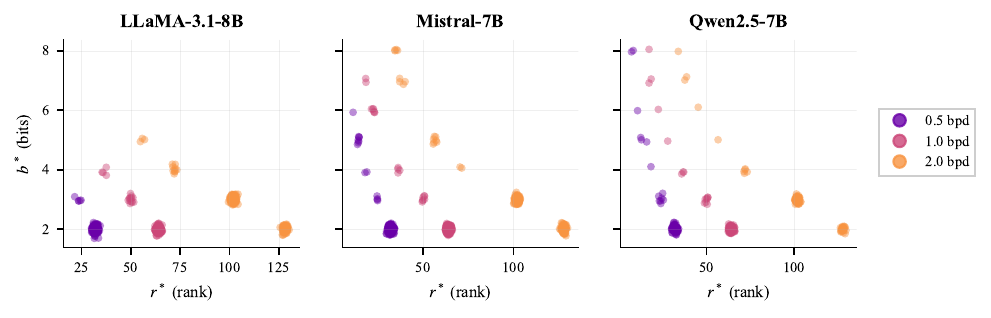}
\caption{Per-head $(r^{\star}, b^{\star})$ selected by C1
across three models, colored by target bpd ($0.5$, $1.0$,
$2.0$).}
\vspace{-1mm}
\label{fig:rstar-div}
\end{figure}

\begin{figure}[H]
\centering
\figorplaceholder[0.85\linewidth]{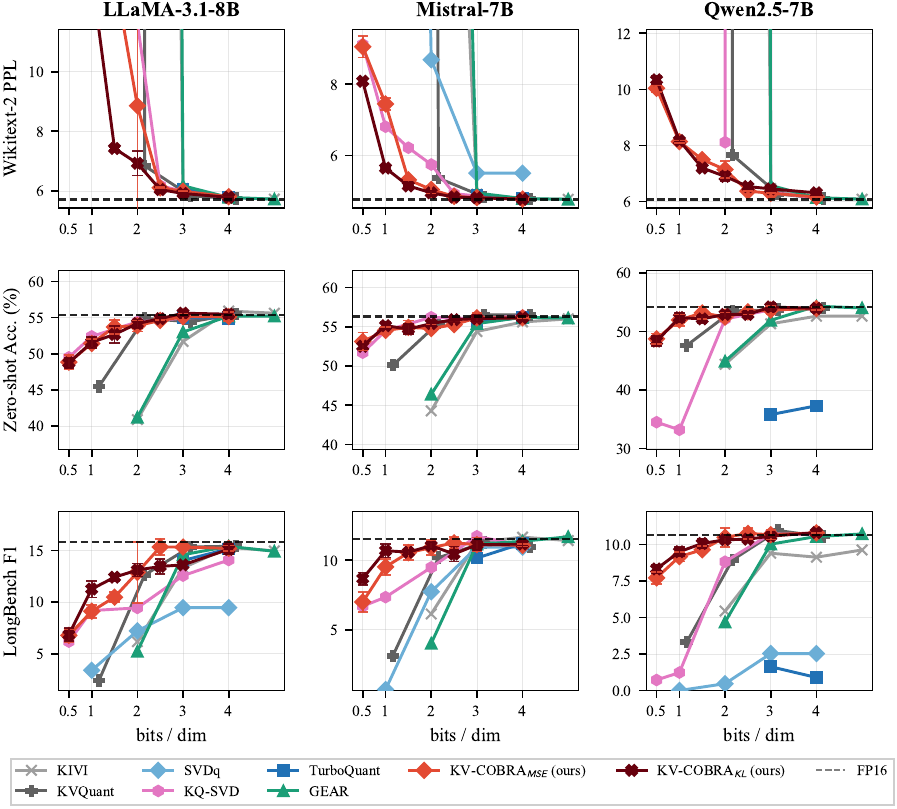}
\caption{Main $k$-only results. Dashed lines mark FP16;
error bars are $\pm 1$~SD over $5$ seeds.}
\label{fig:main-results}
\vspace{-1mm}
\end{figure}

\paragraph{Main comparison.}
Figure~\ref{fig:main-results} sweeps the full $0.5$--$4$~bpd
range across KIVI, KVQuant, GEAR, TurboQuant, SVDq, and
KQ-SVD. At ${\geq}2$~bpd the choice of allocator barely
matters. At low bpd \kvkl{} is the most graceful overall,
holding within range on all three metric averages. KIVI and
GEAR are integer-only and cannot enter the regime, while
SVDq and KQ-SVD degrade sharply: at $1$~bpd \kvkl{} nearly
doubles the LongBench F1 of integer-only baselines.
Table~\ref{tab:main-ablation} zooms in on $0.5$ and
$1.0$~bpd, where PPL is the mean over Wikitext-2, PTB, and
C4 and zero-shot and LongBench are each the mean over five
tasks (multi-seed averages, Appendix~\ref{app:setup}).
\kvmse{} wins on PPL (its direct objective), while \kvkl{}
wins on LongBench and matches \kvmse{} on zero-shot --- the
split predicted by the framing (PPL is
reconstruction-aligned, LongBench attention-aligned).

\begin{table}[H]
\centering
\caption{Metric-group means at $0.5$ and $1.0$~bpd in
$k$-only mode across the three models. PPL is the mean over
Wikitext-2, PTB, and C4 ($5$ seeds). Zero-shot and
LongBench are each the mean over five tasks ($5$ seeds).
\textbf{Bold} marks the better of \kvmse{} and
\kvkl{} per cell. Per-task breakdown is in
Appendix~\ref{app:addl-results}.}
\label{tab:main-ablation}
\footnotesize
\setlength{\tabcolsep}{1.6pt}
\begin{tabular}{l cccccc | cccccc | ccc}
\toprule
 & \multicolumn{6}{c|}{$0.5$~bpd} & \multicolumn{6}{c|}{$1.0$~bpd} & \multicolumn{3}{c}{FP16} \\
 & \multicolumn{2}{c}{LL} & \multicolumn{2}{c}{Mi} & \multicolumn{2}{c|}{Qw} & \multicolumn{2}{c}{LL} & \multicolumn{2}{c}{Mi} & \multicolumn{2}{c|}{Qw} & LL & Mi & Qw \\
\cmidrule(lr){2-3}\cmidrule(lr){4-5}\cmidrule(lr){6-7}\cmidrule(lr){8-9}\cmidrule(lr){10-11}\cmidrule(lr){12-13}\cmidrule(lr){14-16}
Metric & MSE & KL & MSE & KL & MSE & KL & MSE & KL & MSE & KL & MSE & KL & & & \\
\midrule
PPL ($\downarrow$)      & \msdb{31.17}{2.21} & \msd{37.31}{6.44} & \msd{36.50}{1.46} & \msdb{36.24}{0.68} & \msdb{16.33}{0.15} & \msd{17.17}{0.23} & \msd{24.76}{1.14} & \msdb{20.26}{0.87} & \msd{27.18}{0.97} & \msdb{15.90}{0.13} & \msdb{12.97}{0.09} & \msd{13.20}{0.06} & 7.67 & 13.15 & 9.50 \\
Zero-shot ($\uparrow$)  & \msdb{48.8}{1.0} & \msd{48.5}{0.5} & \msdb{53.4}{0.9} & \msd{52.5}{0.3} & \msdb{48.7}{0.9} & \msd{48.5}{0.7} & \msd{51.5}{0.6} & \msdb{51.6}{0.7} & \msd{54.6}{0.5} & \msdb{55.1}{0.5} & \msd{51.7}{0.7} & \msdb{52.5}{0.7} & 55.4 & 56.3 & 54.1 \\
LongBench ($\uparrow$)  & \msdb{6.83}{0.63} & \msd{6.75}{0.58} & \msd{7.02}{0.73} & \msdb{8.73}{0.42} & \msd{7.69}{0.40} & \msdb{8.30}{0.31} & \msd{9.25}{0.54} & \msdb{11.36}{0.77} & \msd{9.55}{0.56} & \msdb{10.72}{0.49} & \msd{9.10}{0.29} & \msdb{9.53}{0.30} & 15.80 & 11.54 & 10.64 \\
\bottomrule
\end{tabular}
\end{table}

\paragraph{Joint K+V compression.}
We hold the total budget $T = K_{\text{bpd}} + V_{\text{bpd}}$
fixed and sweep both sides
(Tables~\ref{tab:kv-joint},~\ref{tab:kv-joint-lb}). The
$V$ side uses a uniform per-head budget because its spectrum
is too flat for cross-head redistribution to help
(Appendix~\ref{app:vside}). On PPL, V-heavy splits dominate
and \kvkl{} takes the per-block optimum on $7$ of $9$
(model, $T$) blocks; conversely, $V_{\text{bpd}}{=}2$
collapses on LLaMA and Mistral across all three methods
(PPL ${>}150$). KQ-SVD's failure on Qwen's
$K_{\text{bpd}}{=}1$ cells (PPL ${\sim}10^{4}$) reflects
its halved KV-head count without per-head allocation.
Table~\ref{tab:kv-joint-lb} reports the same sweep on
LongBench F1.

\begin{table}[H]
\centering
\caption{K+V joint compression on Wikitext-2 PPL
($\downarrow$). Each row scans splits
$(K_{\text{bpd}}, V_{\text{bpd}})$ at fixed total $T$. Both
sides are compressed (V uniform per-head,
Appendix~\ref{app:vside}). The FP16 column is the
uncompressed reference. Shaded cell marks the best PPL
within each (model, $T$) block.}
\label{tab:kv-joint}
\scriptsize
\setlength{\tabcolsep}{2pt}

\begin{tabular}{l l ccc cccc ccccc c}
\toprule
& & \multicolumn{3}{c}{$T = 5$} & \multicolumn{4}{c}{$T = 6$}
& \multicolumn{5}{c}{$T = 7$} & \\
\cmidrule(lr){3-5}\cmidrule(lr){6-9}\cmidrule(lr){10-14}
& Method & $(1,4)$ & $(2,3)$ & $(3,2)$
& $(1,5)$ & $(2,4)$ & $(3,3)$ & $(4,2)$
& $(1,6)$ & $(2,5)$ & $(3,4)$ & $(4,3)$ & $(5,2)$ & FP16 \\
\midrule
\multirow{3}{*}{\textit{LLaMA-3.1-8B}}
& \kvmse{} & 17.35 & 38.84 & 732.1 & 15.79 & 17.88 & 19.17 & 550.4 & 15.21 & 15.48 & 6.69 & 17.26 & 505.7 & \multirow{3}{*}{5.71} \\
& \kvkl{}  & \cellcolor{black!12}14.63 & 24.79 & 490.7 & 13.54 & \cellcolor{black!12}8.86 & 16.55 & 465.1 & 13.35 & 7.74 & \cellcolor{black!12}6.52 & 16.80 & 541.3 & \\
& KQ-SVD   & 26.14 & 50.82 & 528.1 & 15.19 & 15.01 & 27.36 & 559.5 & 13.31 & 12.50 & 6.83 & 25.21 & 551.7 & \\
\midrule
\multirow{3}{*}{\textit{Mistral-7B}}
& \kvmse{} & 7.96 & 8.59 & 285.7 & 7.43 & 5.46 & 7.92 & 326.3 & 7.35 & 5.25 & 4.99 & 7.92 & 311.3 & \multirow{3}{*}{4.77} \\
& \kvkl{}  & \cellcolor{black!12}5.97 & 10.04 & 252.6 & 5.69 & \cellcolor{black!12}5.31 & 8.00 & 305.9 & 5.65 & 5.10 & \cellcolor{black!12}4.98 & 7.89 & 357.1 & \\
& KQ-SVD   & 8.14 & 19.76 & 184.1 & 7.04 & 6.42 & 10.48 & 162.1 & 6.80 & 5.85 & 5.10 & 9.59 & 159.6 & \\
\midrule
\multirow{3}{*}{\textit{Qwen2.5-7B}}
& \kvmse{} & 8.29 & 8.47 & 12.86 & 8.16 & 7.78 & \cellcolor{black!12}6.92 & 11.75 & 8.19 & 7.60 & \cellcolor{black!12}6.43 & 6.78 & 11.57 & \multirow{3}{*}{6.07} \\
& \kvkl{}  & 8.39 & \cellcolor{black!12}7.80 & 16.62 & 8.25 & 7.26 & 7.65 & 13.56 & 8.22 & 7.20 & 7.11 & 7.05 & 12.13 & \\
& KQ-SVD   & 11137. & 10.04 & 19.60 & 10812. & 8.35 & 7.37 & 17.51 & 10932. & 8.12 & 6.47 & 7.21 & 17.25 & \\
\bottomrule
\end{tabular}

\end{table}

\begin{table}[H]
\centering
\caption{K+V joint compression, LongBench F1
($\uparrow$, mean over NarrativeQA, QASPER, MultiFieldQA-en,
HotpotQA, and MuSiQue at $50$ samples per task). Shaded
cell marks the best F1 within each (model, $T$) block.}
\label{tab:kv-joint-lb}
\scriptsize
\setlength{\tabcolsep}{2pt}

\begin{tabular}{l l ccc cccc ccccc c}
\toprule
& & \multicolumn{3}{c}{$T = 5$} & \multicolumn{4}{c}{$T = 6$}
& \multicolumn{5}{c}{$T = 7$} & \\
\cmidrule(lr){3-5}\cmidrule(lr){6-9}\cmidrule(lr){10-14}
& Method & $(1,4)$ & $(2,3)$ & $(3,2)$
& $(1,5)$ & $(2,4)$ & $(3,3)$ & $(4,2)$
& $(1,6)$ & $(2,5)$ & $(3,4)$ & $(4,3)$ & $(5,2)$ & FP16 \\
\midrule
\multirow{3}{*}{\textit{LLaMA-3.1-8B}}
& \kvmse{} & 8.64 & 6.75 & 3.25 & 9.86 & 8.31 & 9.71 & 3.73 & 8.08 & 8.89 & \cellcolor{black!12}13.59 & 8.98 & 3.66 & \multirow{3}{*}{15.80} \\
& \kvkl{}  & \cellcolor{black!12}11.06 & 7.78 & 2.81 & 10.67 & \cellcolor{black!12}11.25 & 8.62 & 3.32 & 12.16 & 12.10 & 12.35 & 8.95 & 3.19 & \\
& KQ-SVD   & 6.21 & 6.63 & 3.59 & 8.23 & 8.14 & 7.02 & 3.42 & 8.78 & 9.23 & 11.55 & 7.66 & 3.38 & \\
\midrule
\multirow{3}{*}{\textit{Mistral-7B}}
& \kvmse{} & 8.42 & 9.07 & 1.60 & 9.71 & 10.25 & 9.13 & 0.95 & 9.35 & 10.80 & \cellcolor{black!12}11.16 & 9.62 & 1.01 & \multirow{3}{*}{11.54} \\
& \kvkl{}  & \cellcolor{black!12}11.13 & 8.41 & 1.95 & 10.04 & \cellcolor{black!12}10.64 & 9.83 & 1.12 & 10.73 & 10.22 & 9.90 & 9.39 & 1.20 & \\
& KQ-SVD   & 6.90 & 5.48 & 1.51 & 8.50 & 9.51 & 8.21 & 1.84 & 8.04 & 10.34 & 10.68 & 9.11 & 2.20 & \\
\midrule
\multirow{3}{*}{\textit{Qwen2.5-7B}}
& \kvmse{} & 9.23 & 9.78 & 8.39 & 8.99 & 9.41 & \cellcolor{black!12}10.35 & 9.24 & 8.89 & 9.36 & 9.82 & 10.16 & 9.07 & \multirow{3}{*}{10.64} \\
& \kvkl{}  & 9.85 & \cellcolor{black!12}10.00 & 8.25 & 8.60 & 10.27 & 10.04 & 8.52 & 9.19 & 10.32 & \cellcolor{black!12}10.48 & 9.84 & 9.35 & \\
& KQ-SVD   & 1.16 & 8.03 & 7.29 & 1.17 & 8.70 & 10.05 & 9.06 & 1.28 & 8.78 & 10.10 & 10.30 & 8.59 & \\
\bottomrule
\end{tabular}
\end{table}

\paragraph{Additional experiments.}
Appendix~\ref{app:addl-results} reports the full per-dataset
PPL and per-task zero-shot/LongBench breakdown
(Tables~\ref{tab:detail-half} and~\ref{tab:detail-one}) and
a per-head attention-primitive wall-clock measurement
showing KV-COBRA is $1.79$--$1.89\times$ faster than FP16-SDPA
at $r{=}16$ across $T\in\{32{\rm k}, 64{\rm k}, 128{\rm k}\}$
(Table~\ref{tab:A4}). Appendix~\ref{app:postrope} covers
pre- versus post-RoPE, and Appendix~\ref{app:ruler} adds
RULER long-context retrieval results.

\section{Discussion and limitations}
\label{sec:discussion}

\paragraph{Allocation is the bottleneck, not the kernel.}
KV-COBRA reuses the same SVD-rotate-and-quantize kernel as
prior work, and every gain we report comes from redirecting
the same bits to better places. The C2 surrogate-optimality
check (Appendix~\ref{app:c2-dp}) reinforces this point.
Replacing our damped proportional update with an exact DP
solver for the sum-distortion minimization reduces
$\sum D^\star$ by $\sim 12$--$16\%$ but actually
\emph{degrades} PPL by $\sim 0.5$ on average. The surrogate
itself is the proxy at fault, not the solver --- which is
exactly why the attention-KL upgrade (Sec.~\ref{sec:kl})
helps where any MSE-side refinement cannot.
Appendix~\ref{app:c2-dp} pins the fault down by replacing the
Bennett estimate with measured per-head distortion, after
which the same exact solver improves on the shipped allocator.
For practitioners,
this means \emph{where} bits are spent --- which head, which
direction --- matters more at low bpd than how they are
quantized.

\paragraph{When does the KL upgrade pay off?}
At ${\geq}2$~bpd $L_2$ and attention-KL agree within noise.
Below $1$~bpd, KL typically improves LongBench F1 even when
MSE wins on PPL (Table~\ref{tab:main-ablation}). KL also takes the per-block
optimum on $7$ of the $9$ (model, $T$) blocks in the K+V
joint sweep (Table~\ref{tab:kv-joint}).

\paragraph{Asymmetric $K$/$V$ budgets.}
$V$-heavy splits dominate the joint $K{+}V$ sweep
(Tables~\ref{tab:kv-joint},~\ref{tab:kv-joint-lb}),
reflecting $V$'s flatter spectrum
(Appendix~\ref{app:vside}). The same $(r^\star, b^\star)$
allocator handles both sides without any $K$-specific
assumption, generalizing the per-head principle to the
asymmetric regime. Appendix~\ref{app:addl-results} breaks the
asymmetry down per task (Table~\ref{tab:kv-pertask}).

\paragraph{Limitations.}
\begin{itemize}[leftmargin=*,itemsep=1pt,topsep=0pt]
\item \textbf{Scope.} Our comparison is restricted to
per-token compressors on the same $(r, b)$ axes. Two
orthogonal families fall outside this scope: token-axis
eviction (H2O~\citep{h2o}, StreamingLLM~\citep{streamingllm}),
which discards entries rather than compressing them, and
entropy coding on top of a quantizer. Both compose with
KV-COBRA in principle, but a fused evaluation is left for
follow-up.
\item \textbf{Wall-clock.} We do not report end-to-end
latency, but Appendix~\ref{app:addl-results} reports
primitive-level attention speedup over FP16-SDPA.
\item \textbf{Calibration stationarity.} $(r^\star, b^\star)$
and budgets are computed once from $32$ Wikitext-2 samples.
The allocation is insensitive to the calibration corpus and
size and transfers across domains
(Appendix~\ref{app:addl-results}), but
refresh-during-generation is not studied.
\item \textbf{Distortion model.} Bennett's leading-order
$q(b)$~\citep{bennett1948} assumes near-Gaussian channels. A
single fixed Hadamard transform carries no distributional
guarantee~\citep{benbasat2026}, so we rely on the empirical
validation of Appendix~\ref{app:gaussianity}. The finite-rate
correction at $b{=}1$ can alter the rank selected by C1
(Appendix~\ref{app:bennett}).
\end{itemize}

\section{Conclusion}

Every rotate-and-quantize KV-cache compressor hides two
allocation decisions --- per-head rank and per-head bit
allocation --- that prior work treats as fixed
hyperparameters.
\textbf{KV-COBRA} (\emph{Co-Optimized Bit-Rank Allocation})
solves them jointly through a two-level decomposition, a
per-head $(r, b)$ optimum (C1) and a cross-head budget
distributor (C2), and adds an attention-KL reordering that
makes both solvers query-aware. The pipeline is gradient-free
and calibration-only, and across LLaMA-3.1-8B, Mistral-7B,
and Qwen2.5-7B it shows the smallest accuracy degradation
among evaluated methods at low bpd. The same allocator
extends to the joint $K{+}V$ regime, where V-heavy splits
typically dominate. The takeaway is
that bit allocation, not the quantization kernel, is where
low-bpd gains live, and stronger rotations or entropy
coding can layer on top without changing the allocator.
Natural extensions include MLA and MoE architectures,
refining the $V$ side, and composing with token-axis
eviction policies.

\bibliographystyle{unsrtnat}
\bibliography{references}

\clearpage
\appendix
\setcounter{figure}{0}
\setcounter{table}{0}
\renewcommand{\thefigure}{A\arabic{figure}}
\renewcommand{\thetable}{A\arabic{table}}
\makeatletter
\renewcommand{\theHfigure}{appendix.\arabic{figure}}
\renewcommand{\theHtable}{appendix.\arabic{table}}
\makeatother

\section{Distortion model and head-level \texorpdfstring{$D(r,b)$}{D(r,b)}}
\label{app:primer}

This section establishes the per-head distortion function
$D(r,b)$ that the main paper uses without proof. We need two
ingredients: (i)~the per-channel quantization MSE
$q(b)\!\cdot\!\sigma^{2}$ (Lemma~\ref{lem:bennett}); and
(ii)~the Eckart--Young identity for the residual MSE of an
orthogonal projection (Lemma~\ref{lem:trunc}).

\subsection{The uniform scalar quantizer after Hadamard rotation}
\label{app:bennett}

KV-COBRA applies a Hadamard rotation to the retained SVD
coordinates before quantization. The rotation equalizes
per-channel variance to the average
$\bar\sigma^2=r^{-1}\sum_{i\leq r}\lambda_i$: with the scaling
$H_r H_r^\top = I$, Walsh--Hadamard rows have squared entries
$1/r$, so for any diagonal covariance
$\mathrm{diag}(\lambda_1,\ldots,\lambda_r)$ every rotated channel
has variance $r^{-1}\sum_j\lambda_j$.
Let $X\sim\mathcal{N}(0,\sigma^{2})$ and let $Q_b$ denote a
uniform mid-tread quantizer at rate $b$ bits, applied to $X$
clipped to a range $[-L\sigma, L\sigma]$ with step
$\Delta = 2L\sigma/2^{b}$ (the loading factor $L$ is
absorbed into $\alpha$ below). Bennett's high-resolution
formula~\citep{bennett1948} then gives
\begin{equation}
\EE\bigl[(X-Q_b(X))^2\bigr]
= \alpha\cdot 2^{-2b}\cdot\sigma^2+o(2^{-2b}),
\label{eq:q-b}
\end{equation}
where $\alpha>0$ depends on the loading factor $L$ (the
standard normalization $\alpha=1/12$ corresponds to
matching the granular term $\Delta^{2}/12$).
Since $D(r,b)$ is the sum of
an $\alpha$-free truncation term and an $\alpha$-proportional
quantization term under the product constraint $rb\leq B$,
the C1 and C2 optimizers only depend on $\alpha$ up to a
$\pm 1$ rank-grid shift that is negligible at $b\gtrsim 2$.
We therefore normalize
\begin{equation}
q(b)\coloneqq \tfrac{1}{12}\cdot 2^{-2b},
\label{eq:q-normalized}
\end{equation}
absorbing $\alpha$ into an overall scale. The Hadamard
rotation produces each channel as a $\pm 1/\sqrt{r}$ weighted
sum of $r$ retained coordinates, the entries of $Y=H_r U_r^\top X$ are uncorrelated and,
for moderately large $r$, approximately
Gaussian~\citep{vershynin2018}. We verify this empirically
(Appendix~\ref{app:gaussianity}): post-rotation channels
track $\mathcal{N}(0,1)$ closely up to the tails.

\begin{lemma}[Per-channel quantization MSE]
\label{lem:bennett}
Let $X\sim\mathcal{N}(0,\sigma^{2})$ be quantized by a uniform
mid-tread quantizer at $b$ bits. Under the normalization
of~\eqref{eq:q-normalized},
$\EE[(X-Q_b(X))^{2}]=q(b)\cdot\sigma^{2}+o(2^{-2b})$.
\end{lemma}

\begin{proof}
Standard Bennett high-resolution
result~\citep{bennett1948, gray1998}; follows from partitioning
$\RR$ into cells of width $\Delta$, Taylor-expanding the
Gaussian density around each cell center, and summing the
leading-order $\Delta^{2}/12$ contributions; absorbing the
$\alpha$ prefactor into~\eqref{eq:q-normalized} gives the
stated form.
\end{proof}

\subsection{Truncation in the SVD basis}

\begin{lemma}[Truncation distortion in the original basis; Eckart--Young~1936~\textnormal{\citep{eckart1936}}, Mirsky~1960~\textnormal{\citep{mirsky1960}}]
\label{lem:trunc}
Let $X$ be any zero-mean random vector in $\RR^{d}$ with finite
second moments, write its covariance as $\Sigma = U\Lambda U^\top$,
and let $U_r\in\RR^{d\times r}$ be the first $r$ columns of $U$.
Define $\hat{X}=U_r U_r^\top X$. Then
\[
\EE\|X-\hat{X}\|^2 = \sum_{i>r}\lambda_i.
\]
\end{lemma}

\begin{proof}
Rotate to the eigenbasis: $Y\coloneqq U^\top X$ has diagonal
covariance $\Lambda$, so $\Var[Y_i]=\lambda_i$ with the $Y_i$
uncorrelated. Projection onto the top-$r$ directions zeros out
$Y_{r+1},\ldots,Y_d$, and $U$ is an isometry, so
$\|X-\hat X\|^2=\sum_{i>r}Y_i^2$. Taking expectations,
$\EE\|X-\hat X\|^2 = \sum_{i>r}\lambda_i$.
\end{proof}

\subsection{Empirical post-Hadamard Gaussianity}
\label{app:gaussianity}

Bennett's formula and the
high-resolution per-channel MSE assumption
(Assumption~\ref{ass:whiten})   require each
post-Hadamard channel marginal to be approximately Gaussian.
A single drawn-once
sign pattern yields no distributional guarantee, and the known
bounds for composed randomized Hadamard transforms vanish only
asymptotically in the dimension~\citep{benbasat2026}, so at our
retained ranks ($r$ as small as $16$) the residual is not
negligible a priori. We therefore
validate this on the calibration set across all
$(\text{layer}, \text{head}, \text{channel})$ triples on three
models: pre-rotation excess kurtosis is heavy-tailed
(median~$\kappa_{\mathrm{pre}}\!\approx\!0.30$),
post-rotation it concentrates sharply around zero
(median~$|\kappa_{\mathrm{post}}|\!<\!0.20$, $95$th
percentile~$<\!1.10$;
Fig.~\ref{fig:kurtosis}). Q--Q plots
(Fig.~\ref{fig:qq}) confirm the
post-rotation channels track $\mathcal{N}(0,1)$ where the
pre-rotation channels show the standard heavy-tailed S-curve.

\begin{figure}[ht!]
\centering
\figorplaceholder[0.75\linewidth]{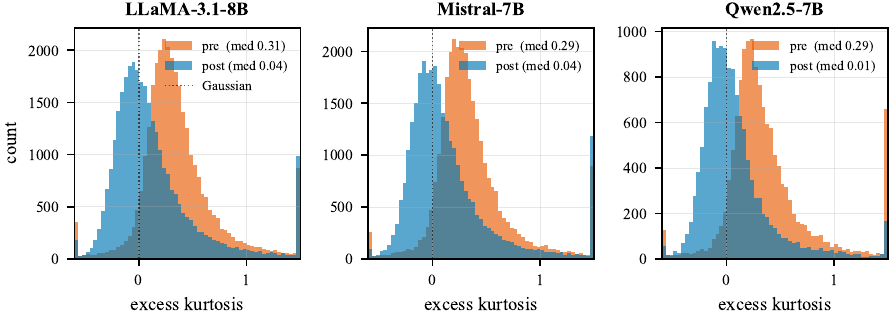}
\caption{Excess kurtosis distribution across all
$(\text{layer},\text{head},\text{channel})$ triples (one panel per
model). Pre-Hadamard (orange) is heavy-tailed
(median~$\kappa\!\approx\!0.30$); post-Hadamard (blue) collapses
near zero (median~$|\kappa|\!<\!0.20$).}
\label{fig:kurtosis}
\end{figure}

\begin{figure}[ht!]
\centering
\figorplaceholder[0.75\linewidth]{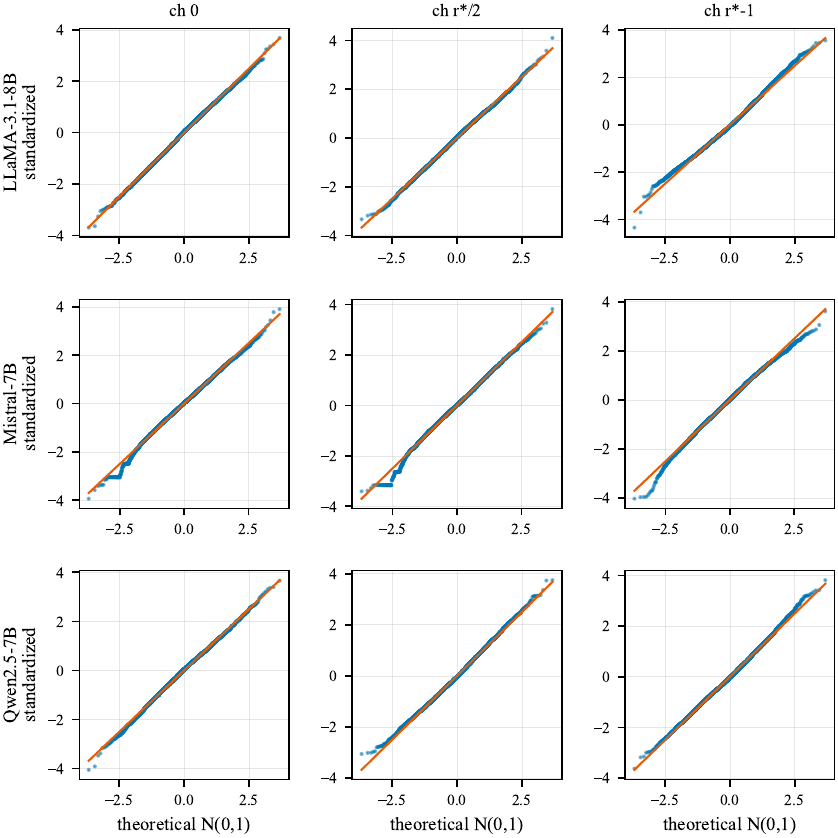}
\caption{Q--Q grid against $\mathcal{N}(0,1)$. Rows: three models;
columns: representative channels (first, mid, last) of one
``median-behavior'' head per model. Post-Hadamard samples track
the diagonal up to $\pm 2.5\sigma$ with mild tail deviation,
consistent with Bennett's high-resolution regime.}
\label{fig:qq}
\end{figure}

\subsection{Putting it together: head-level \texorpdfstring{$D(r,b)$}{D(r,b)}}

The full head-level distortion under the rotate-and-quantize
pipeline (top-$r$ SVD projection followed by per-channel
uniform quantization) is the sum of the two independent
contributions from Lemma~\ref{lem:trunc} (projection residual)
and Lemma~\ref{lem:bennett} (uniform quantizer).

\begin{assumption}[High-resolution per-channel MSE]
\label{ass:whiten}
Let $Y\in\RR^{r}$ be the pre-quantization vector and
$\varepsilon_i\coloneqq Q_b(Y_i)-Y_i$ the per-channel
quantization error. We assume the marginal MSE matches
Bennett's high-resolution formula channel-wise:
$\EE[\varepsilon_i^{2}]=q(b)\cdot\Var(Y_i)+o(q(b))$ for every
$i\in\{1,\ldots,r\}$. This holds exactly as $b\to\infty$ when
$Y_i$ is approximately $\mathcal N(0,\Var(Y_i))$; the
finite-$b$ deviation is absorbed into the $o(q(b))$ remainder.
\end{assumption}

\begin{lemma}[Head-level reconstruction MSE]
\label{lem:head-dist}
Let $X\in\RR^{d}$ be zero-mean with covariance
$\Sigma=U\Lambda U^\top$, eigenvalues
$\lambda_1\geq\cdots\geq\lambda_d\geq 0$. Let $U_r$ be the first
$r$ columns of $U$ and $H_r\in\RR^{r\times r}$ a scaled
Walsh--Hadamard matrix ($H_r H_r^\top = I$). Define the
compressed reconstruction
\[
\hat X\;=\;U_r H_r^\top\,Q_b^{\otimes}\!\bigl(H_r U_r^\top X\bigr),
\]
where $Q_b^{\otimes}$ applies a scalar uniform mid-tread
quantizer at rate $b$ to each of its $r$ inputs. Under
Assumption~\ref{ass:whiten},
\begin{equation}
D(r,b)\;\coloneqq\;\EE\|X-\hat X\|^2
= \sum_{i>r}\lambda_i
\;+\;q(b)\sum_{i\leq r}\lambda_i
\;+\;o\bigl(q(b)\bigr).
\label{eq:head-dist-proved}
\end{equation}
\end{lemma}

\begin{proof}
\emph{Step 1 (orthogonal decomposition).}
Write
\[
X-\hat X
=\underbrace{(I-U_rU_r^\top)X}_{\eqqcolon\,\xi_1\ \in\,\mathrm{range}(U_r)^{\perp}}
+\underbrace{U_rH_r^\top\bigl(H_rU_r^\top X-Q_b^{\otimes}(H_rU_r^\top X)\bigr)}_{\eqqcolon\,\xi_2\ \in\,\mathrm{range}(U_r)}.
\]
The two subspaces are orthogonal deterministically, so
$\|X-\hat X\|^2 = \|\xi_1\|^2+\|\xi_2\|^2$ and
$\EE\|X-\hat X\|^2 = \EE\|\xi_1\|^2+\EE\|\xi_2\|^2$.

\emph{Step 2 (truncation term).}
By Lemma~\ref{lem:trunc},
$\EE\|\xi_1\|^2=\sum_{i>r}\lambda_i$.

\emph{Step 3 (quantization term).}
Let $Y\coloneqq H_rU_r^\top X\in\RR^r$ and
$\varepsilon\coloneqq Q_b^{\otimes}(Y)-Y$. Since
$U_rH_r^\top$ has orthonormal columns,
$\|\xi_2\|^2=\|U_rH_r^\top\varepsilon\|^2=\|\varepsilon\|^2
=\sum_{i=1}^{r}\varepsilon_i^{2}$ (Euclidean-norm definition,
deterministic). By linearity of expectation and
Assumption~\ref{ass:whiten},
\[
\EE\|\xi_2\|^2=\sum_{i=1}^{r}\EE[\varepsilon_i^{2}]
=q(b)\sum_{i=1}^{r}\Var(Y_i)+o(q(b)).
\]

\emph{Step 4 (variance equalization under Hadamard).}
Under the normalization $H_r H_r^\top = I$, Hadamard rows have
squared entries $1/r$, so every diagonal entry of $\Cov(Y)=H_r\,\mathrm{diag}(\lambda_1,\ldots,\lambda_r)\,H_r^\top$
equals $\bar\sigma^{2}=\tfrac{1}{r}\sum_{i\leq r}\lambda_i$, and
\[
\sum_{i=1}^{r}\Var(Y_i)=r\bar\sigma^{2}=\sum_{i\leq r}\lambda_i.
\]

\emph{Step 5 (combine).}
Steps 1--4 yield
$\EE\|X-\hat X\|^2
=\sum_{i>r}\lambda_i+q(b)\sum_{i\leq r}\lambda_i+o(q(b))$.
\end{proof}

Eq.~\eqref{eq:head-dist-proved} is the same $D(r,b)$ that the main
paper uses without proof. Everything about KV-COBRA's per-head
decision is an optimization over this quantity.

\section{The per-head optimum (C1)}
\label{app:c1}

With $D(r,b)$ fully justified, we can now state the discrete-integer
optimality conditions for the per-head joint $(r, b)$ optimizer under a
product budget constraint.

\begin{proposition}[Per-head discrete first-order condition]
\label{thm:c1-kkt}
Let $D(r,b)$ be as in Lemma~\ref{lem:head-dist}, and consider
the integer problem
\[
\min_{\substack{r\in\{r_{\min},r_{\min}{+}2,\ldots,d\}\\
b\in\{b_{\min},\ldots,b_{\max}\}}}
\; D(r, b)
\quad\text{s.t.}\quad r\cdot b \leq B.
\]
Substituting $b = B/r$ reduces this to the one-dimensional
discrete objective $f(r)\coloneqq D(r, B/r)$, whose forward
difference is
\begin{multline*}
\Delta f(r) \;\coloneqq\; f(r{+}2)-f(r)
\;=\;-\bigl(\lambda_{r+1}+\lambda_{r+2}\bigr)
\bigl(1-q(B/(r{+}2))\bigr)\\
+\;\bigl(q(B/(r{+}2))-q(B/r)\bigr)\sum_{i\leq r}\lambda_i.
\end{multline*}
Any integer optimum $r^{\star}$ on this grid satisfies the
two-sided first-order condition
\[
\Delta f(r^{\star}-2)\;\leq\;0\;\leq\;\Delta f(r^{\star}),
\]
with the convention $\Delta f(r_{\min}-2)\coloneqq-\infty$ and
$\Delta f(d)\coloneqq+\infty$ at the boundary.
\end{proposition}

\begin{proof}
Compute
\begin{align*}
f(r{+}2)-f(r)
&=\Bigl[\sum_{i>r+2}\lambda_i\Bigr]
-\Bigl[\sum_{i>r}\lambda_i\Bigr]
+q\bigl(\tfrac{B}{r{+}2}\bigr)\sum_{i\leq r+2}\lambda_i
-q\bigl(\tfrac{B}{r}\bigr)\sum_{i\leq r}\lambda_i\\
&=-(\lambda_{r+1}{+}\lambda_{r+2})
+q\bigl(\tfrac{B}{r{+}2}\bigr)(\lambda_{r+1}{+}\lambda_{r+2})\\
&\quad+\bigl(q\bigl(\tfrac{B}{r{+}2}\bigr)
-q\bigl(\tfrac{B}{r}\bigr)\bigr)\sum_{i\leq r}\lambda_i,
\end{align*}
which rearranges to the stated $\Delta f(r)$. At an integer
minimizer $r^{\star}$ of $f$ on the grid, the step-2 discrete
first-order condition is
$\Delta f(r^{\star}{-}2)\leq 0\leq\Delta f(r^{\star})$ by
definition of a discrete minimum.
\end{proof}

\paragraph{Solver.}
The rank grid $\{r_{\min},r_{\min}{+}2,\ldots,d\}$ is finite
($\leq 32$ candidates at $d{=}128$ with step $2$), so a
brute-force enumeration that sets $b=\mathrm{round}(B/r)$,
clips to $[b_{\min},b_{\max}]$, and returns $\argmin D(r,b)$
attains the global minimum on the integer feasible set. The
step-$2$ rank grid loses no precision in practice: changing
$r$ by $\pm 1$ shifts $b=B/r$ by at most $\approx 0.25$ bits
in our budget range, and we observed the step-$1$ optimum
coinciding with the step-$2$ sub-grid optimum on the heads
we sampled.

\section{Cross-layer sum-distortion minimization (C2)}
\label{app:c2}

\begin{assumption}[Discrete convexity of the head-level envelope]
\label{ass:disc-convex}
For every (layer, head), the integer sequence
$B\mapsto D^{\star}_{\lh}(B)$ of C1-optimal distortions is
non-increasing and discrete-convex:
$D^{\star}_{\lh}(B+1)-2D^{\star}_{\lh}(B)+D^{\star}_{\lh}(B-1)\geq 0$
for every interior integer $B$.
\end{assumption}

Strict discrete-convexity can locally fail at $(r, b)$ grid
transitions, so we treat
Assumption~\ref{ass:disc-convex} as a working approximation;
the empirical impact on the surrogate optimum is small
(\S\ref{sec:c2-vs-dp}).

\begin{proposition}[Cross-layer sub-gradient KKT under
Assumption~\ref{ass:disc-convex}; multi-channel sum-distortion minimization~\textnormal{\citep{cover2006, boyd2004}}]
\label{thm:c2-kkt}
Under Assumption~\ref{ass:disc-convex}, let
$\widetilde{D}_{\lh}^{\star}:\RR_{\geq B_{\min}}\to\RR_{\geq 0}$
be the piecewise-linear envelope of $D^{\star}_{\lh}$
(linear interpolation between consecutive integer grid points).
$\widetilde{D}^{\star}_{\lh}$ is then convex, non-increasing, and
continuous, with sub-differential
\[
\partial\widetilde{D}^{\star}_{\lh}(B)
= \bigl[D^{\star}_{\lh}(B)-D^{\star}_{\lh}(B-1),\;
D^{\star}_{\lh}(B+1)-D^{\star}_{\lh}(B)\bigr]
\qquad\text{at integer }B
\]
(left and right difference quotients, in the standard order
$f'_-(B)\leq f'_+(B)$).
Consider the relaxation
\begin{equation}
\min_{\{B_{\lh}\}\in\RR_{\geq B_{\min}}^{LH}}\;
\sum_{\lh}\widetilde{D}^\star_{\lh}(B_{\lh})
\quad\text{s.t.}\quad
\sum_{\lh} B_{\lh} = B_{\mathrm{tot}}.
\label{eq:c2-primal}
\end{equation}
At any stationary point there exists $\nu^{\star}\geq 0$ such
that for all heads with $B^{\star}_{\lh}>B_{\min}$,
\begin{equation}
-\nu^{\star}\;\in\;\partial\widetilde{D}^\star_{\lh}(B^{\star}_{\lh}),
\label{eq:c2-stationarity}
\end{equation}
and $B^{\star}_{\lh}=B_{\min}$ for the remaining heads
(complementary slackness). On interior points where
$\widetilde{D}^{\star}_{\lh}$ is differentiable, this reduces to
the classical statement that marginal distortions are equalized.
\end{proposition}

\begin{proof}
Under Assumption~\ref{ass:disc-convex} the piecewise-linear
envelope $\widetilde{D}^{\star}_{\lh}$ is convex and continuous
with the stated sub-gradients, so~\eqref{eq:c2-primal} is the
minimization of a sum of closed convex functions subject to one
affine equality and $LH$ affine inequalities. Slater's condition
is trivially satisfied by
$B_{\lh}=B_{\mathrm{tot}}/(LH)\geq B_{\min}$ whenever
$B_{\mathrm{tot}}$ is large enough \citep[Section~5.2.3]{boyd2004}.
The associated Lagrangian is
\[
\mathcal{L}(\{B_{\lh}\};\nu,\{\rho_{\lh}\})
= \sum_{\lh}\widetilde{D}^\star_{\lh}(B_{\lh})
+ \nu\cdot\Bigl(\sum_{\lh}B_{\lh}-B_{\mathrm{tot}}\Bigr)
- \sum_{\lh}\rho_{\lh}\cdot(B_{\lh}-B_{\min}),
\quad \rho_{\lh}\geq 0,
\]
and the sub-gradient optimality conditions at a primal--dual saddle point
are (a) primal feasibility, (b) $\rho_{\lh}\geq 0$, (c)
$\rho_{\lh}(B^{\star}_{\lh}-B_{\min})=0$, and (d)
$0\in\partial\widetilde{D}^\star_{\lh}(B^{\star}_{\lh})+\nu^{\star}-\rho_{\lh}$.
Specializing (d) to $\rho_{\lh}=0$ gives Eq.~\eqref{eq:c2-stationarity}.
\end{proof}

\paragraph{Damped proportional update.}
Eq.~\eqref{eq:c2-stationarity} characterizes the optimum but
does not specify a solver. We use a heuristic that targets a
simpler condition, value equalization
$D^\star_{\lh}=\bar D$ with
$\bar D=\tfrac{1}{LH}\sum_{\lh}D^\star_{\lh}$. This is a
first-order proxy for the KKT condition near uniform: under
$D\propto B^{-\alpha}$, marginal equalization
$\partial D/\partial B = -\nu^{\star}$ becomes
$\alpha D_{\lh}/B_{\lh}=\nu^{\star}$, which reduces to value
equalization $D_{\lh}=\nu^{\star}\bar B/\alpha$ exactly when
all $B_{\lh}\!\approx\!\bar B$. Damping the iteration keeps
budgets close enough to $\bar B$ for this approximation to
remain accurate. Under the empirical scaling $\alpha\!\approx\!2$
(the $2^{-2b}$ falloff of uniform quantization),
heads above the mean require a budget boost roughly
proportional to $\sqrt{D^\star_{\lh}/\bar D}$. Damping by
$\eta=0.3$ around the mean budget $\bar B$ gives the update
\begin{equation}
B_{\lh}^{(t+1)}=\max\!\Big(B_{\min},\;B_{\lh}^{(t)}
+\eta\,\bar B\bigl(\sqrt{D^\star_{\lh}/\bar D}-1\bigr)\Big),
\label{eq:c2-update}
\end{equation}
followed by total-budget renormalization
$B_{\lh}\leftarrow B_{\lh}\cdot B_{\mathrm{tot}}/\sum B_{\lh}$.
At the fixed point $D^\star_{\lh}=\bar D$ for all heads (zero
update at equality), the sign is always correct, and the
per-iteration cost is one C1 call per head ($O(LH)$). On every
tested model the iteration reaches $<\!1\%$ per-head budget
change within five rounds, which we use as the default. The
resulting per-head budgets are non-flat across all three
models (Figure~\ref{fig:xlayer-budget}), confirming that
cross-head heterogeneity is worth exploiting.

\begin{figure}[t]
\centering
\figorplaceholder[0.96\linewidth]{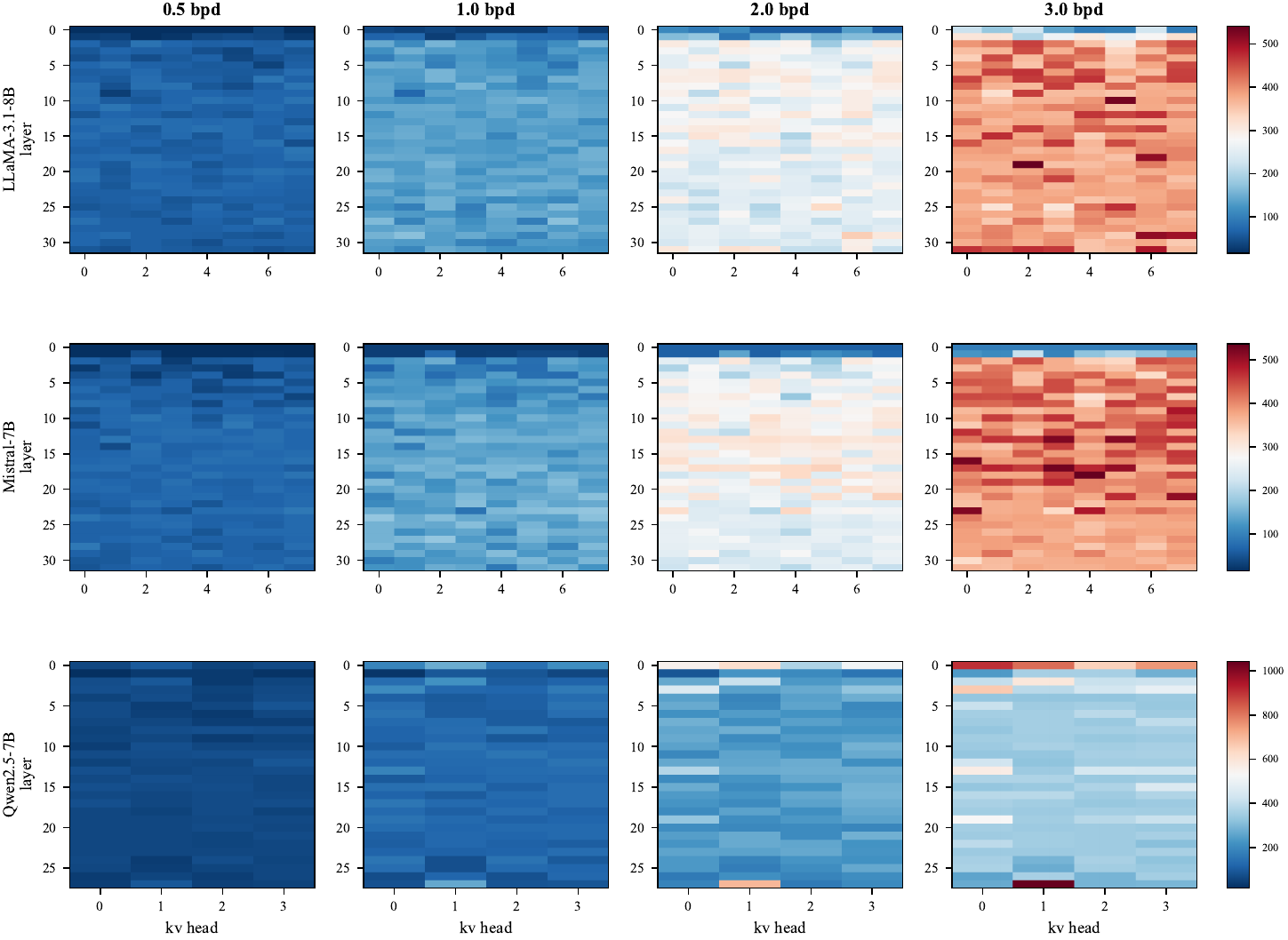}
\caption{C2 budget allocation $B^{\star}_{\lh}$ on $K$ for
LLaMA-3.1-8B, Mistral-7B, Qwen2.5-7B at four target bpds.
Cell color is the per-head bit budget (red high, blue low;
scale normalized per panel). Non-flat pattern $\Rightarrow$
heterogeneity worth exploiting across all three models.}
\label{fig:xlayer-budget}
\end{figure}

\clearpage

\section{C2 solver choice: proportional update vs.\ exact DP}
\label{app:c2-dp}
\label{sec:c2-vs-dp}

The solver used throughout the paper is the damped
proportional update of~\eqref{eq:c2-update} (\S\ref{app:c2}). It
is not surrogate-optimal: it compares $D^\star_{\lh}$ to the
mean $\bar D$ rather than to a subgradient, and the integer
constraint is enforced only by the final rounding. A natural
question is whether a more careful solver would measurably
improve PPL.

\paragraph{Alternative solver: exact DP.}
The integer problem
\[
\min_{\{B_{\lh}\}\in\mathbb{Z}^{LH}}\sum_{\lh}D^\star_{\lh}(B_{\lh})
\quad\text{s.t.}\quad\sum_{\lh}B_{\lh}=B_{\mathrm{tot}},\;
B_{\lh}\in[B_{\min},B_{\max}]
\]
is a separable knapsack and admits an exact
$O(LH\cdot B_{\mathrm{tot}}\cdot C)$ dynamic-programming
solver, where $C=(B_{\max}{-}B_{\min})/\Delta$ is the per-head
candidate count at step $\Delta$. This DP makes \emph{no}
convexity assumption on $D^\star_{\lh}$: it finds the
surrogate-optimal integer allocation even when the head-level
envelope has plateau-then-jump patterns that
Assumption~\ref{ass:disc-convex} rules out.

\paragraph{PPL on real models.}
Table~\ref{tab:c2-dp} reports Wikitext-2 PPL on LLaMA-3.1-8B,
Mistral-7B, and Qwen2.5-7B ($k$-only) for our solver and DP,
with C1 identical across columns.

\begin{table}[H]
\centering
\footnotesize
\caption{C2 solver sensitivity on Wikitext-2 PPL ($k$-only,
single seed). \textbf{Ours} is the damped proportional update
used in the paper. \textbf{DP-optimal} solves the integer
surrogate exactly without any convexity assumption. PPL: cells
where one method beats the other are bolded. Rightmost column:
wall-clock ratio of the C2 step
(${t_{\mathrm{DP}}}/{t_{\mathrm{ours}}}$, single-CPU-thread,
calibration excluded). At $1.0$ bpd, our solver costs
$130$\,ms on LLaMA / Mistral ($256$ heads each) and $58$\,ms
on Qwen ($112$ heads); DP costs $5.3$\,s and $2.3$\,s. The
ratio grows with bpd because DP's table size is
$O(LH\cdot\bar B)$ while our iteration count stays at $5$.}
\label{tab:c2-dp}
\begin{tabular}{l c c c c}
\toprule
Model & bpd & PPL ours & PPL DP-optimal & $t_{\mathrm{DP}}/t_{\mathrm{ours}}$ \\
\midrule
LLaMA-3.1-8B & 0.5 & \textbf{15.56} & 17.02 & 25$\times$ \\
             & 1.0 & \textbf{15.55} & 15.71 & 40$\times$ \\
             & 2.0 &  \textbf{6.78} &  7.38 & 49$\times$ \\
             & 3.0 &  5.98 & \textbf{5.93} & 50$\times$ \\
             & 4.0 &  5.80 & \textbf{5.76} & 51$\times$ \\
\midrule
Mistral-7B   & 0.5 &  7.57 & \textbf{7.32} & 26$\times$ \\
             & 1.0 &  6.16 & \textbf{6.08} & 40$\times$ \\
             & 2.0 & \textbf{5.02} &  5.17 & 51$\times$ \\
             & 3.0 &  \textbf{4.84} &  4.84 & 51$\times$ \\
             & 4.0 &  4.79 & \textbf{4.78} & 50$\times$ \\
\midrule
Qwen2.5-7B   & 0.5 & 12.30 & \textbf{11.92} & 25$\times$ \\
             & 1.0 &  9.63 & \textbf{9.28} & 39$\times$ \\
             & 2.0 & \textbf{6.97} &  8.16 & 49$\times$ \\
             & 3.0 & \textbf{6.35} &  6.36 & 47$\times$ \\
             & 4.0 & \textbf{6.17} &  6.18 & 47$\times$ \\
\bottomrule
\end{tabular}
\end{table}

\paragraph{Takeaway.}
Across the $15$ cells in Table~\ref{tab:c2-dp}, the two solvers
give comparable PPL: the median per-cell gap is $0.01$, the
mean is $0.52$, and individual gaps fall within the
$\pm 0.1$--$0.5$ multi-seed noise.
Wall-clock differs much more: $25$--$51\times$ across the bpd
range (one-time calibration cost). With no measurable PPL gain
from DP we adopt the cheaper solver as the default; the DP
solver ships alongside in the reference code and can be
selected with a single flag.

\paragraph{Where the surrogate deviates.}
A natural follow-up is why surrogate optimality does not
translate into PPL. To find out, we quantized every
(layer, head)'s rotated calibration activations with the
deployed quantizer at $b=1,\ldots,4$ and compared the measured
MSE against the $q(b)\propto 2^{-2b}$ model, normalized at
$b{=}4$ so that only the shape matters
(Table~\ref{tab:c2-shape}). The model is accurate at
$b\geq 3$ but underestimates the true distortion by
$1.7\times$ at $b{=}2$, with a $30$--$40\%$ head-to-head
spread. An allocator compares distortion across heads and bit
widths, so what hurts is error that differs across the options
being compared. An exact solver tends to pick the options
whose cost is most underestimated.

\begin{table}[H]
\centering
\footnotesize
\caption{Measured quantizer MSE divided by the $2^{-2b}$
model (shape only, normalized at $b{=}4$, all three models).
$b{=}1$ is below the deployed $b\geq 2$ floor and is shown for
the trend.}
\label{tab:c2-shape}
\begin{tabular}{l cccc}
\toprule
 & $b=4$ & $b=3$ & $b=2$ & $b=1$ \\
\midrule
median across heads & 1.0 & 1.1 & \textbf{1.7} & 6.9 \\
per-head spread (p95$-$p05)/median & 31\,\% & 33\,\% & \textbf{41\,\%} & 36\,\% \\
\bottomrule
\end{tabular}
\end{table}

\paragraph{Exact DP over measured distortion.}
Replacing the Bennett estimate with measured per-head
distortion tables (quantizing the calibration activations once
per candidate $(r, b)$ adds a few seconds to calibration) and
re-running the same exact DP removes the degradation entirely
(Table~\ref{tab:c2-measured}). The exact solver then improves
on the shipped allocator on all three models, and the solution
becomes nearly seed-invariant. The degradation therefore comes
from the error of the assumed quantizer model, not from the
solver or the objective. Once the estimate is measured instead
of assumed, the same objective and the same solver behave
well.

\begin{table}[H]
\centering
\footnotesize
\caption{Wikitext-2 PPL at $1$~bpd, $k$-only
(mean$\,\pm\,$SD over $3$ seeds).}
\label{tab:c2-measured}
\begin{tabular}{l ccc}
\toprule
 & LLaMA & Mistral & Qwen \\
\midrule
FP16 & 5.71 & 4.77 & 6.07 \\
water-filling + Bennett estimate (shipped) & 16.33 $\pm$ 0.55 & 7.39 $\pm$ 0.13 & 8.14 $\pm$ 0.07 \\
exact DP + Bennett estimate & 17.55 $\pm$ 1.49 & 7.17 $\pm$ 0.19 & 7.98 $\pm$ 0.06 \\
exact DP + measured estimate & \textbf{8.11 $\pm$ 0.02} & \textbf{5.39 $\pm$ 0.01} & \textbf{7.43 $\pm$ 0.01} \\
\bottomrule
\end{tabular}
\end{table}

\section{Effective bpd: per-method overhead derivations}
\label{app:effective-bpd}

Each method's nominal bpd hides a different amount of
per-token metadata. We add an overhead $\Delta$ to every cell,
counting only storage paid \emph{per (channel, token)};
fixed per-head or per-layer bases are excluded as they
amortize away at $T{=}32768$. Head dimension $d=128$, raw bits
(no entropy coding, applied uniformly to all methods).

\paragraph{KIVI~\citep{kivi}.}
Token-axis grouping with $\text{group\_size}=32$; each
(group, channel) stores an fp16 $(\min,\max)$ pair:
$\Delta_{\text{KIVI}} = 32/32 = 1.0$~bpd. (the
configuration used in standard reproductions);
KIVI's $128$-token fp16 residual contributes $\leq 0.07$~bpd
at $T{=}32{,}768$ (negligible at our sequence length).

\paragraph{KVQuant~\citep{kvquant}.}
Per-channel quantization plus top-$\rho{=}1\%$ fp16 outliers,
following KVQuant's own convention (excluding CSR index overhead):
$\Delta_{\text{KVQuant}} = \rho\cdot 16 = 0.16$~bpd. Sink tokens
add $80/T$~bpd, negligible at $T\geq 1000$.

\paragraph{GEAR~\citep{gear}.}
Group-wise quantization ($\text{group\_size}=64$) plus
rank-$r{=}4$ residual $Q_mP^\top$ with $Q_m\in\RR^{T\times r}$:
group meta $32/64 = 0.5$, residual row $r\cdot 16/d = 0.5$,
basis $P$ assumed shared across tokens (FlexGen-style
convention; per-token $P$ would double the residual cost).
$\Delta_{\text{GEAR}} = 1.0$~bpd. (excluding the sparse
outlier component, which adds further overhead).

\paragraph{TurboQuant~\citep{turboquant}.}
Random rotation + Lloyd--Max; metadata is a per-channel
scale (fixed per-head). $\Delta_{\text{TurboQuant}} = 0$.

\paragraph{SVDq~\citep{svdq}, KQ-SVD~\citep{kqsvd}, KV-COBRA (ours).}
All three store a fixed SVD basis (per-layer for SVDq, per-head
otherwise) and cache only the quantized latent of $r$ channels
at $b$ bits. The per-(channel, token) cost is exactly the
nominal $r\cdot b/d$; $\Delta = 0$ for all three.

\paragraph{Summary.}
KIVI and GEAR each pay $+1.0$~bpd on top of their nominal $b$;
KVQuant pays $+0.16$; rotation-based methods pay $0$. A KIVI or
GEAR cell at nominal $2$~bpd sits at the same effective
operating point as our $3$~bpd.

\section{Scope: \texorpdfstring{$V$}{V}-side compression is structurally limited}
\label{app:vside}

\paragraph{Flatter spectra.} Power-law fits
$\lambda_i\propto i^{-\alpha}$ give median
$\alpha_V/\alpha_K\approx 0.46$ across the three models
(Fig.~\ref{fig:kv-spectrum}, top row). The flatter $V$
spectrum pins C1 at the ceiling $r^\star{=}d$ on essentially
every head at $2$~bpd (bottom row), while on $K$ the optimum
is spread far below $d$ so C1 has room to trade rank for bit
width.

\paragraph{C2 and KL reordering add nothing.} With
per-head spectra nearly identical, each head's
$D^\star(B)$ curve is nearly the same, so C2 has no free
energy to redistribute. Empirically, the KL reordering
provides no measurable gain on $V$, consistent with the
flat-spectrum behavior.

\begin{figure}[H]
\centering
\figorplaceholder[0.85\linewidth]{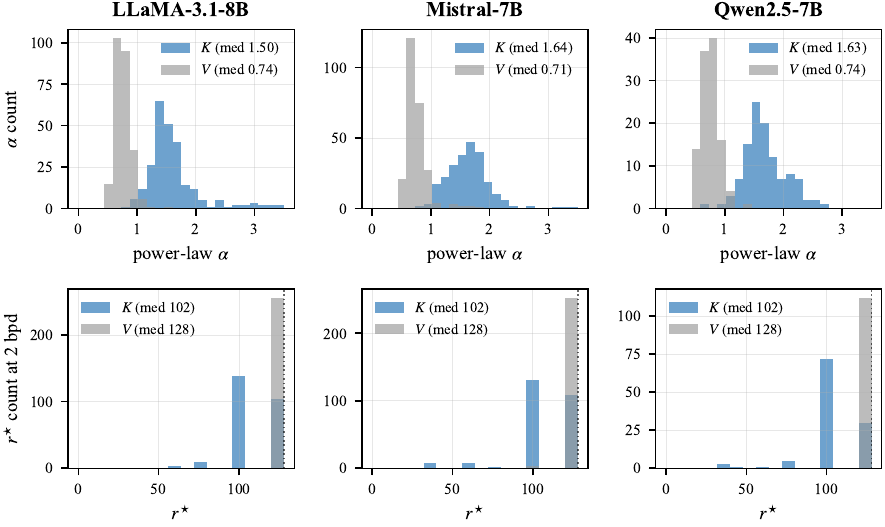}
\caption{$V$ is structurally flatter than $K$.
\textbf{Top:} per-(layer, head) power-law exponents
$\alpha$; median $\alpha_V/\alpha_K\approx 0.46$.
\textbf{Bottom:} C1-selected $r^\star$ at $2$~bpd. On $V$,
$r^\star=d{=}128$ on nearly every head (dotted line); on $K$
the distribution is broad.}
\label{fig:kv-spectrum}
\end{figure}

\paragraph{Empirical check.}
Figure~\ref{fig:kv-vs-v} compares \kvc{} (C1 only)
on $K$-only vs.\ $V$-only at matched bit budgets.
On a log scale the $V$-only curve sits orders of magnitude
above the $K$-only curve through sub-$3$~bpd and converges to
FP16 only around $4$--$5$~bpd: $V$ is harder to compress at
low rates because truncating its flat spectrum is expensive.

\begin{figure}[H]
\centering
\figorplaceholder[0.85\linewidth]{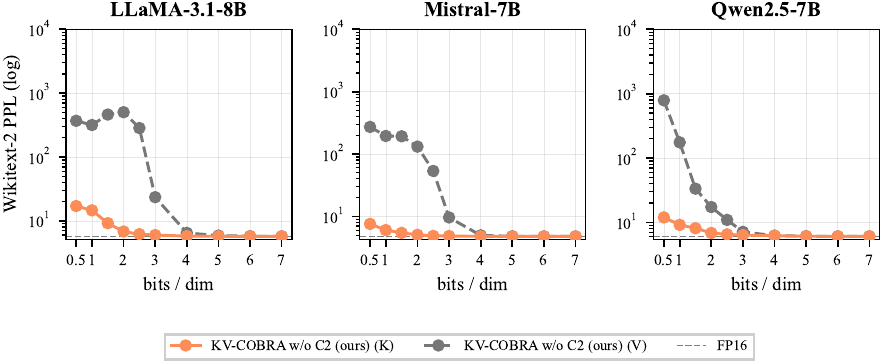}
\caption{$K$-only vs.\ $V$-only Wikitext-2 PPL (log)
for \kvc{} on three models. Orange solid: $K$-only;
grey dashed: $V$-only; dashed horizontal: FP16. $V$ converges
to FP16 only near $4$--$5$~bpd.}
\label{fig:kv-vs-v}
\end{figure}

\paragraph{What we do on $V$.} Putting the three observations
together, the $V$ side reduces to plain bit allocation. The
flat spectrum makes rank truncation lossy for any
non-trivial savings, the per-head $D^\star(B)$ curves are
similar enough that cross-head redistribution buys nothing
(so C2 is redundant), and the KL reordering coincides with
the MSE one. We therefore reuse the same kernel as on $K$
but skip C2, and the resulting pipeline is a plain
rotate-and-quantize on the $V$ projection. The Hadamard
rotation still equalizes per-channel variance and so is kept,
which is consistent with the joint $K{+}V$ table where the
$V$-side bits behave like a uniform quantizer
(Tables~\ref{tab:kv-joint},~\ref{tab:kv-joint-lb}). Because
the $V$ side has no per-head allocation to optimize, we do
not feature $V$-only as a primary evaluation track---the
joint $K{+}V$ sweep is the meaningful test, since it lets
the $K$ side absorb the rank budget that $V$ cannot use.

\section{Post-RoPE experiments}
\label{app:postrope}

We ran every tested method under both pre-RoPE and post-RoPE
compression with the same calibration protocol. Post-RoPE requires
a new harness because the attention forward has to be monkey-
patched to inject the compressor after the rotary embedding is
applied to $K$. Pre-RoPE remains uniformly better
(Figure~\ref{fig:postrope}) because the per-head covariance is
position-stationary pre-RoPE and position-dependent post-RoPE: a
single static SVD basis captures the structure in the first case
but not the second.

The mechanism is worth spelling out. After RoPE, each position
stores a differently rotated copy of the key distribution, so
a static basis faces the average over positions,
$\EE_{p}\!\left[R(p)\,\Sigma\,R(p)^{\top}\right]$, which is
much closer to isotropic than $\Sigma$ itself, with the
high-frequency rotation planes flattening first. The low-rank
structure the method relies on fades, so any static basis pays
more truncation loss at the same rank. This matches the
per-channel structure mixing reported by
KVQuant~\citep{kvquant}. Position-dependent bases could track
the rotation in principle, at the cost of the one-shot
calibration that keeps the method simple. We leave this as
future work.

\begin{figure}[ht!]
\centering
\figorplaceholder[0.85\linewidth]{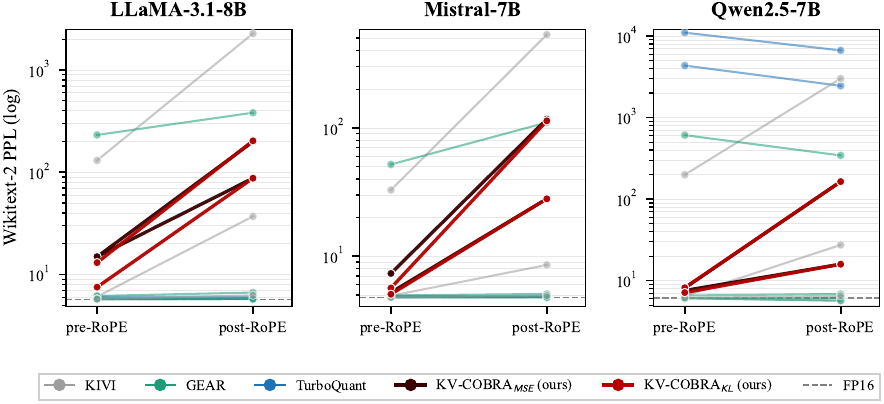}
\caption{Pre-RoPE vs.\ post-RoPE PPL at the canonical bit rates,
per baseline. Pre-RoPE is uniformly better: the per-head
covariance is position-stationary and admits a good shared basis.}
\label{fig:postrope}
\end{figure}

\section{Experimental setup in detail}
\label{app:setup}

This appendix expands Section~\ref{sec:setup} with the details
omitted there for space.

\paragraph{Models.} LLaMA-3.1-8B~\citep{llama3};
Mistral-7B-v0.3~\citep{mistral}; Qwen2.5-7B-Instruct~\citep{qwen25}.
All three models use grouped-query attention and are loaded
in \texttt{fp16} with their default tokenizers, providing
three independently trained $7$--$8$B backbones from
different labs.

\paragraph{Calibration.} Each calibration run passes $32$ text
samples of $1{,}024$ tokens each from the Wikitext-2~\citep{wikitext}
training split through a single forward pass with hooks on every
$k_{\mathrm{proj}}$ and $v_{\mathrm{proj}}$ output. The per-head
activations are centered and eigendecomposed in
\texttt{float32}. Calibration is deterministic given the seed
and the tokenized data; no gradient steps are taken at any
point.

\paragraph{PPL evaluation.} Perplexity is measured on
Wikitext-2~\citep{wikitext}, the Penn Treebank (PTB)~\citep{ptb}, and the C4
English subset~\citep{c4}, each at $32{,}768$ evaluation tokens using
a sliding window of length $2{,}048$ and stride $512$. These
settings match the ones used by the KIVI~\citep{kivi} and KVQuant~\citep{kvquant}
evaluation scripts.

\paragraph{LongBench and zero-shot.} We report mean F1 on the
LongBench~\citep{longbench} English subset at $50$ samples
per task on NarrativeQA~\citep{narrativeqa},
QASPER~\citep{qasper}, MultiFieldQA-en (introduced
by~\citep{longbench}), HotpotQA~\citep{hotpotqa}, and
MuSiQue~\citep{musique}. Zero-shot accuracy is reported on
ARC-Challenge~\citep{arc}, HellaSwag~\citep{hellaswag},
PIQA~\citep{piqa}, WinoGrande~\citep{winogrande}, and
MMLU~\citep{mmlu} at $200$ samples per task. All scores in
the paper are per-task means.

\paragraph{Baselines and bit-width support.} Every baseline
reported in the paper is ported directly from the reference
implementation provided by the original authors, with only
enough glue code to share our \texttt{KVCompressor} dispatch.
TurboQuant requires precomputed Lloyd--Max codebooks per
bit-width and is restricted to $\{3, 4\}$~bpd as in the
official repository. KIVI and KVQuant use uniform integer
quantization with per-(group, channel) statistics that is
well-defined at any integer bit-width; we evaluate them at
$\{1, 2, 3, 4\}$~bpd in our sweep, while their original
papers focus on $\{2, 4\}$~bpd (KIVI) and $\{2, 3, 4\}$~bpd
(KVQuant).

\paragraph{Bit-rate sweep.} Integer-only baselines run at
the bit-widths supported by their official implementations
(see above). \kvmse{} and \kvkl{} run at $\{0.5, 1\}$~bpd
in $k$-only mode for the multi-seed ablation
(Table~\ref{tab:main-ablation}); the bit-rate scan in
Fig.~\ref{fig:main-results} additionally covers
$\{1.5, 2, 2.5, 3, 4\}$~bpd. The $K{+}V$ joint sweep runs at
the $(K_{\text{bpd}}, V_{\text{bpd}})$ splits of
Tables~\ref{tab:kv-joint} and~\ref{tab:kv-joint-lb}. We do
not run $K{+}V$ symmetric cells because the decomposition
treats the two sides independently.

\paragraph{Seeds and variance.} The pipeline (calibration,
compression, evaluation) is deterministic given
\texttt{torch}, \texttt{numpy}, and the Walsh--Hadamard sign
seed. Single-seed runs use a fixed seed.
Table~\ref{tab:main-ablation} and Fig.~\ref{fig:main-results}
report mean\,$\pm$\,1~SD over $5$ seeds across all three models. Across these
multi-seed cells the median relative SD is $0.51\%$ and the
maximum is $9.06\%$ (at $0.5$~bpd).

\section{Additional empirical results}
\label{app:addl-results}

\paragraph{Per-dataset / per-task breakdown at $0.5$ and $1.0$~bpd.}
Tables~\ref{tab:detail-half} and \ref{tab:detail-one} expand the body-level summary
(Table~\ref{tab:main-ablation}) into per-dataset PPL and
per-task zero-shot / LongBench breakdowns, highlighting where
each objective wins. \kvkl{} takes most zero-shot tasks on
LLaMA, and dominates LongBench especially on Qasper,
MultiFieldQA, and NarrativeQA on Mistral. \kvmse{} retains PPL
advantage on LLaMA Wikitext-2 and on most Qwen metrics.

\paragraph{Wall-clock at the primitive level.}
KV-COBRA stores the cache as packed $b$-bit codes plus a
fold-in of the Hadamard rotation into $V_r$, so the per-head
attention primitive at decode dequantizes and applies the
inverse rotation in a single fused pass. Table~\ref{tab:A4}
reports the speedup of this primitive over FP16-SDPA at the
same shapes ($H_q=32$, $d=128$), under CUDA graphs on an
RTX 6000 Ada. At the retained-rank choice used in the paper
($r=16$), KV-COBRA is $1.79$--$1.89\times$ faster than FP16
across $T\in\{32{\rm k},64{\rm k},128{\rm k}\}$; the speedup
holds up to $r=64$ and only inverts at $r=128$ where the
quantized path's BMM cost overwhelms the bandwidth savings.

\begin{table}[tpb]
\centering
\caption{Per-dataset / per-task breakdown at $0.5$~bpd ($k$-only).
PPL ($\downarrow$), zero-shot (\%, $\uparrow$), LongBench F1 ($\uparrow$),
all aggregate $5$ seeds; cells report mean{\tiny$\,\pm\,$std}.}
\label{tab:detail-half}
\footnotesize
\setlength{\tabcolsep}{2.5pt}
\begin{tabular}{l cccccc}
\toprule
 & \multicolumn{2}{c}{LLaMA} & \multicolumn{2}{c}{Mistral} & \multicolumn{2}{c}{Qwen} \\
\cmidrule(lr){2-3}\cmidrule(lr){4-5}\cmidrule(lr){6-7}
Metric & MSE & KL & MSE & KL & MSE & KL \\
\midrule
\multicolumn{7}{l}{\emph{Perplexity} ($\downarrow$)} \\
\quad Wikitext-2 & \textbf{16.29{\tiny$\,\pm\,$0.94}} & 17.60{\tiny$\,\pm\,$0.72} & 9.05{\tiny$\,\pm\,$0.31} & \textbf{8.08{\tiny$\,\pm\,$0.07}} & \textbf{10.06{\tiny$\,\pm\,$0.12}} & 10.37{\tiny$\,\pm\,$0.13} \\
\quad PTB & \textbf{30.75{\tiny$\,\pm\,$3.11}} & 34.74{\tiny$\,\pm\,$1.06} & \textbf{81.88{\tiny$\,\pm\,$4.66}} & 86.20{\tiny$\,\pm\,$2.10} & \textbf{19.50{\tiny$\,\pm\,$0.27}} & 20.37{\tiny$\,\pm\,$0.34} \\
\quad C4 & \textbf{46.46{\tiny$\,\pm\,$4.10}} & 59.60{\tiny$\,\pm\,$19.14} & 18.57{\tiny$\,\pm\,$1.20} & \textbf{14.44{\tiny$\,\pm\,$0.20}} & \textbf{19.45{\tiny$\,\pm\,$0.16}} & 20.76{\tiny$\,\pm\,$0.30} \\
\quad \textit{mean} & \textbf{31.17{\tiny$\,\pm\,$2.21}} & 37.31{\tiny$\,\pm\,$6.44} & 36.50{\tiny$\,\pm\,$1.46} & \textbf{36.24{\tiny$\,\pm\,$0.68}} & \textbf{16.33{\tiny$\,\pm\,$0.15}} & 17.17{\tiny$\,\pm\,$0.23} \\
\midrule
\multicolumn{7}{l}{\emph{Zero-shot accuracy} (\%, $\uparrow$)} \\
\quad ARC-C & \textbf{32.0{\tiny$\,\pm\,$1.4}} & 30.9{\tiny$\,\pm\,$1.0} & \textbf{36.1{\tiny$\,\pm\,$1.2}} & 35.2{\tiny$\,\pm\,$1.1} & 33.0{\tiny$\,\pm\,$1.6} & \textbf{35.5{\tiny$\,\pm\,$1.4}} \\
\quad HellaSwag & \textbf{46.9{\tiny$\,\pm\,$0.8}} & 46.6{\tiny$\,\pm\,$0.7} & \textbf{49.7{\tiny$\,\pm\,$0.4}} & 49.7{\tiny$\,\pm\,$0.5} & \textbf{48.2{\tiny$\,\pm\,$0.9}} & 48.1{\tiny$\,\pm\,$0.9} \\
\quad PIQA & \textbf{74.3{\tiny$\,\pm\,$1.4}} & 73.3{\tiny$\,\pm\,$1.4} & \textbf{76.3{\tiny$\,\pm\,$1.1}} & 76.2{\tiny$\,\pm\,$0.6} & \textbf{73.3{\tiny$\,\pm\,$1.7}} & 72.7{\tiny$\,\pm\,$1.9} \\
\quad WinoGrande & 61.1{\tiny$\,\pm\,$3.2} & \textbf{61.7{\tiny$\,\pm\,$0.9}} & \textbf{69.0{\tiny$\,\pm\,$2.7}} & 66.9{\tiny$\,\pm\,$1.6} & \textbf{61.6{\tiny$\,\pm\,$1.9}} & 58.2{\tiny$\,\pm\,$2.8} \\
\quad MMLU & 29.5{\tiny$\,\pm\,$1.8} & \textbf{30.2{\tiny$\,\pm\,$1.8}} & \textbf{35.9{\tiny$\,\pm\,$1.0}} & 34.5{\tiny$\,\pm\,$1.0} & 27.5{\tiny$\,\pm\,$1.1} & \textbf{28.1{\tiny$\,\pm\,$1.0}} \\
\quad \textit{mean} & \textbf{48.8{\tiny$\,\pm\,$1.0}} & 48.5{\tiny$\,\pm\,$0.5} & \textbf{53.4{\tiny$\,\pm\,$0.9}} & 52.5{\tiny$\,\pm\,$0.3} & \textbf{48.7{\tiny$\,\pm\,$0.9}} & 48.5{\tiny$\,\pm\,$0.7} \\
\midrule
\multicolumn{7}{l}{\emph{LongBench F1} ($\uparrow$)} \\
\quad NarrativeQA & \textbf{8.16{\tiny$\,\pm\,$1.31}} & 7.99{\tiny$\,\pm\,$1.16} & 4.25{\tiny$\,\pm\,$0.75} & \textbf{6.12{\tiny$\,\pm\,$1.22}} & 6.99{\tiny$\,\pm\,$0.41} & \textbf{7.36{\tiny$\,\pm\,$0.43}} \\
\quad Qasper & 9.61{\tiny$\,\pm\,$0.54} & \textbf{9.69{\tiny$\,\pm\,$1.27}} & 10.77{\tiny$\,\pm\,$0.51} & \textbf{11.17{\tiny$\,\pm\,$1.25}} & 9.61{\tiny$\,\pm\,$0.46} & \textbf{11.27{\tiny$\,\pm\,$0.91}} \\
\quad MultiFieldQA & \textbf{11.51{\tiny$\,\pm\,$1.25}} & 11.44{\tiny$\,\pm\,$0.58} & 14.38{\tiny$\,\pm\,$2.32} & \textbf{18.37{\tiny$\,\pm\,$1.03}} & 15.94{\tiny$\,\pm\,$0.49} & \textbf{16.09{\tiny$\,\pm\,$0.42}} \\
\quad HotpotQA & \textbf{3.27{\tiny$\,\pm\,$1.19}} & 2.76{\tiny$\,\pm\,$0.44} & 3.67{\tiny$\,\pm\,$1.40} & \textbf{5.02{\tiny$\,\pm\,$1.40}} & 3.88{\tiny$\,\pm\,$0.83} & \textbf{4.72{\tiny$\,\pm\,$0.74}} \\
\quad MuSiQue & 1.61{\tiny$\,\pm\,$0.47} & \textbf{1.88{\tiny$\,\pm\,$0.39}} & 2.02{\tiny$\,\pm\,$0.46} & \textbf{2.98{\tiny$\,\pm\,$0.53}} & 2.01{\tiny$\,\pm\,$0.54} & \textbf{2.06{\tiny$\,\pm\,$0.19}} \\
\quad \textit{mean} & \textbf{6.83{\tiny$\,\pm\,$0.63}} & 6.75{\tiny$\,\pm\,$0.58} & 7.02{\tiny$\,\pm\,$0.73} & \textbf{8.73{\tiny$\,\pm\,$0.42}} & 7.69{\tiny$\,\pm\,$0.40} & \textbf{8.30{\tiny$\,\pm\,$0.31}} \\
\bottomrule
\end{tabular}
\end{table}

\begin{table}[tpb]
\centering
\caption{Per-dataset / per-task breakdown at $1.0$~bpd ($k$-only).
PPL ($\downarrow$), zero-shot (\%, $\uparrow$), LongBench F1 ($\uparrow$),
all aggregate $5$ seeds; cells report mean{\tiny$\,\pm\,$std}.}
\label{tab:detail-one}
\footnotesize
\setlength{\tabcolsep}{2.5pt}
\begin{tabular}{l cccccc}
\toprule
 & \multicolumn{2}{c}{LLaMA} & \multicolumn{2}{c}{Mistral} & \multicolumn{2}{c}{Qwen} \\
\cmidrule(lr){2-3}\cmidrule(lr){4-5}\cmidrule(lr){6-7}
Metric & MSE & KL & MSE & KL & MSE & KL \\
\midrule
\multicolumn{7}{l}{\emph{Perplexity} ($\downarrow$)} \\
\quad Wikitext-2 & 16.45{\tiny$\,\pm\,$0.51} & \textbf{13.27{\tiny$\,\pm\,$0.60}} & 7.47{\tiny$\,\pm\,$0.15} & \textbf{5.65{\tiny$\,\pm\,$0.03}} & \textbf{8.14{\tiny$\,\pm\,$0.05}} & 8.17{\tiny$\,\pm\,$0.03} \\
\quad PTB & 25.01{\tiny$\,\pm\,$1.68} & \textbf{21.12{\tiny$\,\pm\,$1.17}} & 59.38{\tiny$\,\pm\,$2.65} & \textbf{32.86{\tiny$\,\pm\,$0.33}} & \textbf{14.79{\tiny$\,\pm\,$0.15}} & 15.07{\tiny$\,\pm\,$0.12} \\
\quad C4 & 32.83{\tiny$\,\pm\,$1.27} & \textbf{26.39{\tiny$\,\pm\,$1.08}} & 14.70{\tiny$\,\pm\,$0.61} & \textbf{9.18{\tiny$\,\pm\,$0.07}} & \textbf{15.99{\tiny$\,\pm\,$0.15}} & 16.35{\tiny$\,\pm\,$0.10} \\
\quad \textit{mean} & 24.76{\tiny$\,\pm\,$1.14} & \textbf{20.26{\tiny$\,\pm\,$0.87}} & 27.18{\tiny$\,\pm\,$0.97} & \textbf{15.90{\tiny$\,\pm\,$0.13}} & \textbf{12.97{\tiny$\,\pm\,$0.09}} & 13.20{\tiny$\,\pm\,$0.06} \\
\midrule
\multicolumn{7}{l}{\emph{Zero-shot accuracy} (\%, $\uparrow$)} \\
\quad ARC-C & 34.8{\tiny$\,\pm\,$1.2} & \textbf{34.9{\tiny$\,\pm\,$1.2}} & 39.1{\tiny$\,\pm\,$1.6} & \textbf{40.1{\tiny$\,\pm\,$2.0}} & 38.1{\tiny$\,\pm\,$1.2} & \textbf{38.2{\tiny$\,\pm\,$2.0}} \\
\quad HellaSwag & \textbf{46.8{\tiny$\,\pm\,$0.7}} & 46.6{\tiny$\,\pm\,$0.8} & 50.1{\tiny$\,\pm\,$0.5} & \textbf{50.9{\tiny$\,\pm\,$0.6}} & 50.8{\tiny$\,\pm\,$1.1} & \textbf{51.6{\tiny$\,\pm\,$0.7}} \\
\quad PIQA & 75.9{\tiny$\,\pm\,$0.7} & \textbf{76.1{\tiny$\,\pm\,$1.4}} & 78.1{\tiny$\,\pm\,$1.3} & \textbf{78.3{\tiny$\,\pm\,$0.8}} & 77.6{\tiny$\,\pm\,$1.2} & \textbf{77.7{\tiny$\,\pm\,$0.9}} \\
\quad WinoGrande & 64.6{\tiny$\,\pm\,$2.0} & \textbf{67.2{\tiny$\,\pm\,$2.2}} & \textbf{71.0{\tiny$\,\pm\,$1.5}} & 71.0{\tiny$\,\pm\,$1.7} & 64.3{\tiny$\,\pm\,$2.8} & \textbf{66.2{\tiny$\,\pm\,$1.7}} \\
\quad MMLU & \textbf{35.4{\tiny$\,\pm\,$1.1}} & 33.1{\tiny$\,\pm\,$1.2} & 34.4{\tiny$\,\pm\,$1.2} & \textbf{35.2{\tiny$\,\pm\,$0.5}} & 27.5{\tiny$\,\pm\,$1.3} & \textbf{28.8{\tiny$\,\pm\,$1.7}} \\
\quad \textit{mean} & 51.5{\tiny$\,\pm\,$0.6} & \textbf{51.6{\tiny$\,\pm\,$0.7}} & 54.6{\tiny$\,\pm\,$0.5} & \textbf{55.1{\tiny$\,\pm\,$0.5}} & 51.7{\tiny$\,\pm\,$0.7} & \textbf{52.5{\tiny$\,\pm\,$0.7}} \\
\midrule
\multicolumn{7}{l}{\emph{LongBench F1} ($\uparrow$)} \\
\quad NarrativeQA & 9.38{\tiny$\,\pm\,$0.40} & \textbf{11.21{\tiny$\,\pm\,$0.39}} & 6.37{\tiny$\,\pm\,$1.22} & \textbf{8.34{\tiny$\,\pm\,$1.16}} & 7.46{\tiny$\,\pm\,$0.36} & \textbf{7.65{\tiny$\,\pm\,$0.34}} \\
\quad Qasper & 9.60{\tiny$\,\pm\,$1.25} & \textbf{11.12{\tiny$\,\pm\,$0.78}} & 11.57{\tiny$\,\pm\,$1.47} & \textbf{12.57{\tiny$\,\pm\,$1.37}} & 12.02{\tiny$\,\pm\,$0.71} & \textbf{13.25{\tiny$\,\pm\,$0.64}} \\
\quad MultiFieldQA & 18.70{\tiny$\,\pm\,$1.24} & \textbf{22.28{\tiny$\,\pm\,$1.85}} & 20.45{\tiny$\,\pm\,$1.23} & \textbf{21.73{\tiny$\,\pm\,$1.06}} & 18.70{\tiny$\,\pm\,$0.59} & \textbf{19.24{\tiny$\,\pm\,$0.82}} \\
\quad HotpotQA & 5.56{\tiny$\,\pm\,$1.39} & \textbf{8.36{\tiny$\,\pm\,$2.03}} & 5.82{\tiny$\,\pm\,$0.80} & \textbf{7.18{\tiny$\,\pm\,$0.33}} & 4.88{\tiny$\,\pm\,$0.35} & \textbf{5.10{\tiny$\,\pm\,$0.54}} \\
\quad MuSiQue & 3.04{\tiny$\,\pm\,$0.41} & \textbf{3.83{\tiny$\,\pm\,$0.43}} & 3.53{\tiny$\,\pm\,$0.32} & \textbf{3.77{\tiny$\,\pm\,$0.59}} & \textbf{2.45{\tiny$\,\pm\,$0.17}} & 2.40{\tiny$\,\pm\,$0.17} \\
\quad \textit{mean} & 9.25{\tiny$\,\pm\,$0.54} & \textbf{11.36{\tiny$\,\pm\,$0.77}} & 9.55{\tiny$\,\pm\,$0.56} & \textbf{10.72{\tiny$\,\pm\,$0.49}} & 9.10{\tiny$\,\pm\,$0.29} & \textbf{9.53{\tiny$\,\pm\,$0.30}} \\
\bottomrule
\end{tabular}
\end{table}

\begin{table}[tpb]
\centering
\footnotesize
\caption{Per-layer attention speedup at $H_q=32$, $d=128$, CUDA-graph captured. Each cell is FP16-SDPA latency / KV-COBRA latency at the same $T$ (i.e. $>1\times$ means KV-COBRA is faster); \textbf{bold} marks $>1\times$.}
\label{tab:A4}
\begin{tabular}{c cc cc cc cc}
\toprule
 & \multicolumn{2}{c}{$r=16$} & \multicolumn{2}{c}{$r=32$} & \multicolumn{2}{c}{$r=64$} & \multicolumn{2}{c}{$r=128$} \\
\cmidrule(lr){2-3} \cmidrule(lr){4-5} \cmidrule(lr){6-7} \cmidrule(lr){8-9}
$T$ & $b=2$ & $b=4$ & $b=2$ & $b=4$ & $b=2$ & $b=4$ & $b=2$ & $b=4$ \\
\midrule
32k  & \textbf{1.83$\times$} & \textbf{1.84$\times$} & \textbf{1.77$\times$} & \textbf{1.73$\times$} & \textbf{1.40$\times$} & \textbf{1.31$\times$} & 0.44$\times$ & 0.32$\times$ \\
64k  & \textbf{1.79$\times$} & \textbf{1.85$\times$} & \textbf{1.82$\times$} & \textbf{1.74$\times$} & \textbf{1.48$\times$} & \textbf{1.40$\times$} & 0.43$\times$ & 0.36$\times$ \\
128k & \textbf{1.89$\times$} & \textbf{1.88$\times$} & \textbf{1.85$\times$} & \textbf{1.76$\times$} & \textbf{1.49$\times$} & \textbf{1.38$\times$} & 0.40$\times$ & 0.33$\times$ \\
\bottomrule
\end{tabular}
\end{table}

\paragraph{Task-dependence of the K/V asymmetry.}
Table~\ref{tab:kv-pertask} lists the best
$(K_{\text{bpd}}, V_{\text{bpd}})$ split of \kvkl{} for each
LongBench task at total budgets $5/6/7$~bpd, per model. The
V-heavy preference holds at the level of individual tasks, not
only the suite mean. The per-task best split is V-heavy in
$40$ of $45$ (model, task, budget) cells, and all five
exceptions are QASPER or MultiFieldQA, the most
extraction-flavored tasks, where the best split drifts from
V-heavy through balanced to K-heavy as the budget grows. The
pattern matches the two error paths. $K$ distortion perturbs
\emph{addressing} (logit noise through the softmax, most
damaging when attention must be peaked), while $V$ distortion
perturbs \emph{content} (entering the output linearly and
degrading smoothly). At $V_{\text{bpd}}{=}2$ the flat $V$
spectrum leaves no low-error directions to drop, which is why
the collapse is abrupt.

\begin{table}[tpb]
\centering
\footnotesize
\caption{Best $(K_{\text{bpd}}, V_{\text{bpd}})$ split of
\kvkl{} per LongBench task at total budgets $5$\,/\,$6$\,/\,$7$~bpd.
\textbf{Bold} marks the five non-V-heavy cells.}
\label{tab:kv-pertask}
\setlength{\tabcolsep}{3pt}
\begin{tabular}{l lll}
\toprule
Task & LLaMA-3.1-8B & Mistral-7B & Qwen2.5-7B \\
\midrule
NarrativeQA & (1,4) / (1,5) / (1,6) & (1,4) / (2,4) / (3,4) & (1,4) / (2,4) / (3,4) \\
QASPER & (1,4) / (2,4) / (3,4) & (1,4) / \textbf{(3,3)} / \textbf{(4,3)} & (2,3) / \textbf{(3,3)} / \textbf{(4,3)} \\
MultiFieldQA & (1,4) / (2,4) / (1,6) & (1,4) / (2,4) / (3,4) & (2,3) / \textbf{(3,3)} / (3,4) \\
HotpotQA & (1,4) / (1,5) / (2,5) & (1,4) / (2,4) / (2,5) & (1,4) / (1,5) / (2,5) \\
MuSiQue & (2,3) / (2,4) / (2,5) & (1,4) / (2,4) / (3,4) & (2,3) / (2,4) / (3,4) \\
\bottomrule
\end{tabular}
\end{table}

\paragraph{Layer-position sensitivity.}
C1/C2 weight every head purely by the magnitude of its
spectral distortion. A natural question is whether equal
distortion is equally harmful at every depth. To test this we
apply the full $1$-bpd \kvkl{} allocation to one layer
quartile at a time (all other layers FP16) and normalize each
quartile's Wikitext-2 PPL increase by its share of the
distortion surrogate at the deployed $(r, b)$ (mean over three
calibration seeds). Depth does matter. Per unit of surrogate
distortion, the most sensitive quartile does
$2.5$--$5.9\times$ more damage than the least sensitive one,
and the profile is model-dependent (LLaMA is most sensitive in
the first quartile, Mistral in the early half, Qwen in the
mid-to-late layers). Errors also compound across depth.
Compressing all layers at once costs more than the sum of the
per-quartile costs ($+7.72$ vs.\ $\Sigma{=}4.03$ PPL on
LLaMA). The current objective captures part of the depth
effect through each layer's spectrum but does not reweight
across layers. Folding a depth weight measured at calibration
time into C2 is a natural extension.

\paragraph{Calibration robustness.}
We recomputed the full C1{+}C2 solution under five alternative
calibrations and compared each against the paper's $n{=}32$
WikiText-2 draw (Table~\ref{tab:calib-alloc}). Per-head
budgets, chosen bit widths, and downstream scores all stay in
a narrow band, and changing the calibration size perturbs the
allocation about as much as re-drawing a same-size set.
Calibrating on Python source instead of WikiText-2 moves
HumanEval pass@1 by at most $1.6$ points, so the allocation
also transfers across domains.

\begin{table}[tpb]
\centering
\footnotesize
\caption{Calibration robustness vs.\ the paper's $n{=}32$
WikiText-2 calibration. Allocation columns give ranges over
the three models and $\bpd\in\{0.5,1,2\}$. LongBench is the
five-task F1 mean of \kvkl{} at $1$~bpd ($50$ samples per
task).}
\label{tab:calib-alloc}
\setlength{\tabcolsep}{4pt}
\begin{tabular}{l ccc ccc}
\toprule
 & \multicolumn{3}{c}{Allocation vs.\ paper} & \multicolumn{3}{c}{LongBench F1} \\
\cmidrule(lr){2-4}\cmidrule(lr){5-7}
Calibration & budget corr. & identical $b^\star$ & mean $|\Delta r^\star|$ & LLaMA & Mistral & Qwen \\
\midrule
(reference, $n{=}32$ WikiText-2) & --- & --- & --- & 10.8 & 10.6 & 9.4 \\
Disjoint WikiText-2 draw & 0.992--1.000 & 97--100\,\% & 0.45--1.17 & 9.6 & 10.0 & 9.4 \\
PTB & 0.984--0.998 & 94--100\,\% & 0.76--3.05 & 12.1 & 10.5 & 8.9 \\
C4 & 0.990--0.999 & 94--100\,\% & 0.42--2.30 & 9.0 & 9.8 & 9.4 \\
$n=8$ ($4\times$ less data) & 0.991--0.999 & 91--100\,\% & 0.30--3.31 & 10.6 & 10.1 & 9.2 \\
$n=128$ ($4\times$ more data) & 0.993--0.999 & 96--100\,\% & 0.37--1.79 & 10.3 & 9.8 & 9.4 \\
\bottomrule
\end{tabular}
\end{table}

\paragraph{Scale and family check.}
Table~\ref{tab:scale-family} repeats the three-metric
comparison at $1$~bpd ($k$-only) on LLaMA-3.1-70B,
Qwen2.5-72B, and Mixtral-8x7B (a mixture-of-experts model
whose attention and KV structure match Mistral-7B, with MoE
only in the feed-forward blocks). The pattern from the
$7$--$8$B models carries over. \kvkl{} stays closest to FP16
on every metric in all three families, while a different
integer-grid baseline collapses on perplexity in each.
One-time calibration at 70B takes $18.3$~s on four $48$~GB
GPUs ($14.6$~s calibration forward passes, $2.2$~s
eigendecomposition of all $640$ KV heads, $0.4$~s C1{+}C2
solve with Hadamard fold-in, $1.1$~s KL reordering), so
per-head SVD is not a bottleneck at scale.

\begin{table}[tpb]
\centering
\footnotesize
\setlength{\tabcolsep}{4pt}
\caption{Scale and family check at $1$~bpd ($k$-only). PPL
($\downarrow$) is the mean over Wikitext-2/PTB/C4. Zero-shot
($\uparrow$) is the mean over ARC-C, HellaSwag, PIQA,
WinoGrande, and MMLU ($200$ samples each, with
mean$\,\pm\,$SD over $3$ seeds for KV-COBRA and seed-invariant
single runs otherwise). LongBench ($\uparrow$) is the
five-task F1 mean ($25$ samples per task).
Shaded cell marks the best compressed method per column.}
\label{tab:scale-family}
\begin{tabular}{ll ccc}
\toprule
Model & Method & PPL & Zero-shot & LongBench \\
\midrule
\multirow{7}{*}{\textit{LLaMA-3.1-70B}}
& FP16 & 5.38 & 60.4 & 14.3 \\
& \kvkl{} & \cellcolor{black!12}6.97 & 58.4 $\pm$ 0.6 & \cellcolor{black!12}13.5 \\
& \kvmse{} & 9.50 & \cellcolor{black!12}58.5 $\pm$ 0.9 & 10.5 \\
& KQ-SVD & 9.79 & 57.0 & 11.2 \\
& KVQuant & $>\!10^{3}$ & 55.9 & 5.1 \\
& KIVI & 163.8 & 52.6 & 3.7 \\
& GEAR & 245.1 & 51.4 & 5.6 \\
\midrule
\multirow{7}{*}{\textit{Qwen2.5-72B}}
& FP16 & 6.27 & 58.0 & 16.3 \\
& \kvkl{} & \cellcolor{black!12}7.62 & 56.7 $\pm$ 0.7 & \cellcolor{black!12}16.8 \\
& \kvmse{} & 7.97 & \cellcolor{black!12}57.1 $\pm$ 0.5 & 14.2 \\
& KQ-SVD & 7.69 & 54.9 & 13.6 \\
& KVQuant & 910.1 & 54.4 & 4.1 \\
& KIVI & 17.8 & 52.2 & 10.0 \\
& GEAR & 50.7 & 50.9 & 8.1 \\
\midrule
\multirow{7}{*}{\textit{Mixtral-8x7B}}
& FP16 & 7.95 & 53.4 & 11.3 \\
& \kvkl{} & \cellcolor{black!12}13.19 & \cellcolor{black!12}52.3 $\pm$ 0.8 & \cellcolor{black!12}11.3 \\
& \kvmse{} & 13.36 & 51.8 $\pm$ 0.1 & 10.0 \\
& KQ-SVD & 27.02 & 52.2 & 7.7 \\
& KVQuant & 323.5 & 49.1 & 3.6 \\
& KIVI & 58.0 & 44.7 & 6.6 \\
& GEAR & 96.0 & 43.9 & 7.0 \\
\bottomrule
\end{tabular}
\end{table}

\clearpage
\section{Long-context retrieval on RULER}
\label{app:ruler}

This appendix collects the RULER~\citep{ruler} results, mean
accuracy over the $13$ subtasks at $T{=}4096$ ($k$-only).
RULER probes retrieval directly and separates the two
objectives far more than PPL does. The KL reordering's margin
over MSE is largest here (Table~\ref{tab:ruler}), consistent
with the attention-quality framing of Section~\ref{sec:kl}.
At $2$~bpd \kvkl{} reaches $92.5/90.6/90.9$ (single seed),
within $0.7$--$3.6$ points of FP16.
Figure~\ref{fig:ruler-sweep} sweeps the full bit-rate range.

\begin{figure}[H]
\centering
\figorplaceholder[\linewidth]{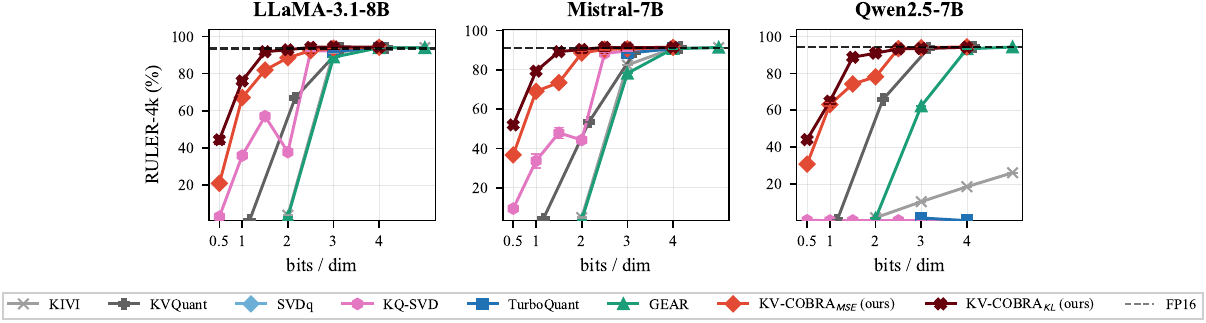}
\caption{RULER-4k mean accuracy versus bit rate ($k$-only).
Dashed lines mark FP16. Error bars are $\pm 1$~SD where
multi-seed data exist.}
\label{fig:ruler-sweep}
\end{figure}

\begin{table}[H]
\centering
\footnotesize
\caption{RULER-4k mean accuracy ($\uparrow$, mean over the
$13$ RULER subtasks at $T{=}4096$, $k$-only,
mean$\,\pm\,$SD over $5$ seeds).}
\label{tab:ruler}
\begin{tabular}{l cc cc cc}
\toprule
 & \multicolumn{2}{c}{LLaMA-3.1-8B} & \multicolumn{2}{c}{Mistral-7B} & \multicolumn{2}{c}{Qwen2.5-7B} \\
\cmidrule(lr){2-3}\cmidrule(lr){4-5}\cmidrule(lr){6-7}
bpd & \kvmse{} & \kvkl{} & \kvmse{} & \kvkl{} & \kvmse{} & \kvkl{} \\
\midrule
$0.5$ & 20.9 $\pm$ 1.1 & \textbf{44.3 $\pm$ 2.3} & 36.7 $\pm$ 1.1 & \textbf{52.0 $\pm$ 0.7} & 30.7 $\pm$ 0.8 & \textbf{44.0 $\pm$ 0.7} \\
$1.0$ & 67.1 $\pm$ 1.7 & \textbf{76.0 $\pm$ 0.7} & 69.1 $\pm$ 1.2 & \textbf{79.3 $\pm$ 1.4} & 63.0 $\pm$ 0.7 & \textbf{64.9 $\pm$ 1.2} \\
\midrule
FP16 & \multicolumn{2}{c}{93.5} & \multicolumn{2}{c}{91.3} & \multicolumn{2}{c}{94.5} \\
\bottomrule
\end{tabular}
\end{table}

\paragraph{Joint K+V splits on RULER.}
Table~\ref{tab:kv-joint-ruler} repeats the joint $K{+}V$
sweep of Table~\ref{tab:kv-joint} on RULER-4k. The $K$/$V$
asymmetry \emph{reverses} relative to PPL and LongBench. The
best cells are balanced to K-heavy, and starving $K$
($K_{\text{bpd}}{=}1$) costs $15$--$30$ points, whereas on
PPL the V-heavy splits dominate. This matches the two error
paths of Section~\ref{sec:discussion}. Retrieval needs peaked
attention on a single token, so the extra bits buy the most
on the addressing side.

\begin{table}[H]
\centering
\caption{K+V joint compression, RULER-4k mean accuracy
($\uparrow$). Same splits as Table~\ref{tab:kv-joint}. Shaded
cell marks the best score within each (model, $T$) block.
KQ-SVD's near-zero Qwen scores mirror its PPL divergence in
Table~\ref{tab:kv-joint}.}
\label{tab:kv-joint-ruler}
\scriptsize
\setlength{\tabcolsep}{1.5pt}
\begin{tabular}{l l ccc cccc ccccc c}
\toprule
& & \multicolumn{3}{c}{$T = 5$} & \multicolumn{4}{c}{$T = 6$}
& \multicolumn{5}{c}{$T = 7$} & \\
\cmidrule(lr){3-5}\cmidrule(lr){6-9}\cmidrule(lr){10-14}
& Method & $(1,4)$ & $(2,3)$ & $(3,2)$
& $(1,5)$ & $(2,4)$ & $(3,3)$ & $(4,2)$
& $(1,6)$ & $(2,5)$ & $(3,4)$ & $(4,3)$ & $(5,2)$ & FP16 \\
\midrule
\multirow{3}{*}{\textit{LLaMA-3.1-8B}}
& \kvmse{} & 67.33 & 87.61 & 92.07 & 68.65 & 88.33 & 93.71 & 93.11 & 68.60 & 88.24 & 93.72 & 93.88 & 93.43 & \multirow{3}{*}{94.11} \\
& \kvkl{}  & 76.53 & 92.35 & \cellcolor{black!12}93.60 & 75.86 & 92.77 & \cellcolor{black!12}94.02 & 93.53 & 76.80 & 92.82 & \cellcolor{black!12}94.31 & 94.24 & 93.50 & \\
& KQ-SVD   & 33.68 & 33.73 & 87.98 & 33.73 & 36.20 & 91.46 & 92.98 & 34.88 & 36.16 & 92.08 & 93.71 & 93.03 & \\
\midrule
\multirow{3}{*}{\textit{Mistral-7B}}
& \kvmse{} & 67.87 & 86.74 & 90.06 & 67.11 & 87.47 & 90.75 & 90.93 & 67.63 & 87.28 & 90.49 & 91.50 & 90.75 & \multirow{3}{*}{91.52} \\
& \kvkl{}  & 79.66 & 88.99 & \cellcolor{black!12}90.34 & 80.40 & 89.46 & \cellcolor{black!12}91.21 & 91.04 & 80.86 & 89.35 & 91.08 & 91.22 & 91.15 & \\
& KQ-SVD   & 35.58 & 42.12 & 86.37 & 35.86 & 42.86 & 89.33 & 90.04 & 36.34 & 43.47 & 89.99 & \cellcolor{black!12}91.53 & 90.42 & \\
\midrule
\multirow{3}{*}{\textit{Qwen2.5-7B}}
& \kvmse{} & 62.32 & 78.01 & \cellcolor{black!12}93.18 & 61.79 & 78.05 & \cellcolor{black!12}93.89 & 93.85 & 62.76 & 78.26 & 93.94 & 94.22 & 93.92 & \multirow{3}{*}{94.50} \\
& \kvkl{}  & 63.31 & 91.41 & 90.79 & 64.79 & 91.09 & 93.26 & 93.16 & 63.95 & 91.28 & 93.66 & \cellcolor{black!12}94.41 & 93.66 & \\
& KQ-SVD   & 0.06 & 0.05 & 0.05 & 0.10 & 0.00 & 0.01 & 0.00 & 0.11 & 0.01 & 0.03 & 0.02 & 1.43 & \\
\bottomrule
\end{tabular}
\end{table}


\end{document}